\documentclass{article}

\usepackage[preprint]{neurips_2020}
\usepackage[utf8]{inputenc} % allow utf-8 input
\usepackage[T1]{fontenc}    % use 8-bit T1 fonts
\usepackage{url}            % simple URL typesetting
\usepackage{booktabs}       % professional-quality tables
\usepackage{amsfonts}       % blackboard math symbols
\usepackage{nicefrac}       % compact symbols for 1/2, etc.
\usepackage{microtype}      % microtypography
\usepackage{graphicx}
\usepackage{subfig}
\usepackage{amssymb,amsbsy,amsmath,amsfonts,amssymb,amscd}
\usepackage{amsthm}
\usepackage{comment}
\usepackage{verbatim}
\usepackage[dvipsnames]{xcolor}
\usepackage{enumitem}
\usepackage{algorithm,algpseudocode}
\usepackage{hyperref}       % hyperlinks
\usepackage{setspace}
\newtheorem{theorem}{Theorem}[section]
\newtheorem{proposition}{Proposition}[section]

\newtheorem{lemma}{Lemma}[section]
\newtheorem{assum}{Assumption}[section]
\newtheorem{cor}{Corollary}[section]

\newcommand{\EE}{\mathbb{E}}

\newcommand{\wx}{\tilde{x}}

\newcommand{\wsx}{\widetilde{X}}

\newcommand{\rd}{\,\mathrm{d}}

\newcommand{\R}{\mathbb{R}}

\title{Random Coordinate Langevin Monte Carlo}

\author{%
  Zhiyan Ding\\%\thanks{Use footnote for providing further information
  Department of Mathematics\\
  University of Wisconsin-Madison\\
  Madison, WI 53706, USA \\
  \texttt{zding49@math.wisc.edu} \\
  \And
   Qin Li \\
  Department of Mathematics \\
  University of Wisconsin-Madison\\
  Madison, WI 53706, USA \\
  \texttt{qinli@math.wisc.edu} \\
  \And
  Jianfeng Lu\\
Mathematics Department\\
Duke University \\
Durham, NC 27708, USA \\
\texttt{jianfeng@math.duke.edu} \\
\AND
Stephen J. Wright\\
Computer Sciences Department\\
University of Wisconsin-Madison\\
Madison, WI 53706, USA \\
\texttt{swright@cs.wisc.edu}
}

\begin{document}

\maketitle

\begin{abstract}
Langevin Monte Carlo (LMC) is a popular Markov chain Monte Carlo sampling method. One drawback is that it requires the computation of the full gradient at each iteration, an expensive operation if the dimension of the problem is high. We propose a new sampling method: Random Coordinate LMC (RC-LMC). At each iteration, a single coordinate is randomly selected to be updated by a multiple of the partial derivative along this direction plus noise, and all other coordinates remain untouched. We investigate the total complexity of RC-LMC and compare it with the classical LMC for log-concave probability distributions. When the gradient of the log-density is Lipschitz, RC-LMC is less expensive than the classical LMC if the log-density is highly skewed for high dimensional problems, and when both the gradient and the Hessian of the log-density are Lipschitz, RC-LMC is always cheaper than the classical LMC, by a factor proportional to the square root of the problem dimension. In the latter case, our estimate of complexity is sharp with respect to the dimension. 
\end{abstract}

\section{Introduction}
Monte Carlo sampling plays an important role in machine learning \citep{MCMCforML} and Bayesian statistics. In applications, the need for sampling is found in atmospheric science \citep{FABIAN198117}, epidemiology~\citep{COVID_travel}, petroleum engineering~\citep{PES}, in the form of data assimilation~\citep{Reich2011}, volume computation~\citep{Convexproblem} and bandit optimization~\citep{ATTS}.

% Sampling is a core problem in Bayesian inference, and plays an important role in data assimilation~\citep{Reich2011}, and machine learning \citep{MCMCforML}. \jl{not sure why we highlight data assimilation here? Perhaps put it in application? I suggest write it as ``Monte Carlo sampling plays an important role in Bayesian statistics and machine learning'' ... (note that ``sampling'' means something different for statisticians; better say Monte Carlo sampling).} In applications, the need for sampling is found in atmospheric science \citep{FABIAN198117}, epidemiology~\citep{OGinverse}, inverse problems in petroleum engineering~\citep{PES}, volume computation~\citep{Convexproblem} and bandit optimization~\citep{ATTS}. \jl{I don't think ``volume computation'' is a real application as the others.} \ql{I agree with your comments here -- can you rewrite this paragraph? Otherwise there is another round of back-and-forth}

In many of these applications, the dimension of the problem is extremely high. For example, for weather prediction, one measures the current state temperature and moisture level, to infer the flow in the air, before running the Navier--Stokes equations into the near future~\citep{Evensen}. In a global numerical weather prediction model, the degrees of freedom in the air flow can be as high as $10^9$. Another example is from epidemiology: When a disease is spreading, one measures the everyday new infection cases to infer the transmission rate in different regions. On a county-level modeling, one treats $3,141$ different counties in the US separately, and the parameter to be inferred has a dimension of at least $3,141$~\citep{COVID_travel}.

%\ql{One category of examples is the PDE (partial-differential-equation)-based inverse problem such as weather prediction. With quantities such as the temperature and the moisture level measured at the current state, one infers the flow in the air, and runs the Navier–Stokes equations to predict the weather in the near future. The quantities to be inferred is the air flow, and in a global numerical weather prediction model, the degrees of freedom is as high as $10^9$. Another example is from epidemiology: when a disease is spreading, one measures the everyday new infection cases to infer the transmission rate in different regions. There are $3,141$ counties in the US, and the to-be-inferred parameter is $3,141$ dimensional.}

In this work, we focus on Monte Carlo sampling of log-concave probability distributions on $\mathbb{R}^d$, meaning the probability density can be written as $p(x) \propto e^{-f(x)}$ where a $f(x)$ is a convex function. The goal is to generate (approximately) i.i.d.~samples according to the target probability distribution with density $p(x)$. Several sampling frameworks have been proposed in the literature, including importance sampling and sequential Monte Carlo~\citep{IM1989,Neal2001,del2006sequential}; ensemble methods~\citep{Reich2011,Iglesias_2013}; Markov chain Monte Carlo (MCMC) ~\citep{Robert2004}, including Metropolis-Hasting based MCMC (MH-MCMC) \citep{doi:10.1063/1.1699114,MCSH,roberts1996}; Gibbs samplers \citep{Geman1984,10.2307/2685208}; and Hamiltonian Monte Carlo~\citep{Neal1993,DUANE1987216}. 
Langevin Monte Carlo (LMC)~\citep{doi:10.1063/1.436415,PARISI1981378,roberts1996} is a popular MCMC method that has received intense attention in recent years due to progress in the non-asymptotic analysis of its convergence properties \citep{durmus2017,doi:10.1111/rssb.12183, DALALYAN20195278, durmus2018analysis}.
%\jl{Several papers here are not really LMC: Chen-Fox-Guestrin is Hamiltonian Monte Carlo, Cheng2017 and Eberle are underdamped; and we miss a couple like the the works by Eric Moulines.} \jl{Shall we start another sentence to list related methods like HMC and underdamped Langevin etc.? I think it is better to be clear on references.} \ql{I prefer not distinguishing them here. we need to be short!}

Denoting by $x^m$ the location of the sample at $m$-th iteration, LMC obtains the next location as follows:
\begin{equation}\label{eqn:update_ujn}
x^{m+1}=x^m-\nabla f(x^m)h+\sqrt{2h}\xi^{m}_d\,,
\end{equation}
where $h$ is the time stepsize, and $\xi^{m}_d$ is drawn i.i.d. from $\mathcal{N}(0,I_d)$, where $I_d$ denotes identity matrix of size $d\times d$.
LMC can be viewed as the Euler-Maruyama discretization of the following stochastic differential equation (SDE):
\begin{equation}\label{eqn:Langevin}
\rd X_t=-\nabla f(X_t)\rd t+\sqrt{2}\rd B_t\,,\\
\end{equation}
where $B_t$ is a $d$-dimensional Brownian motion with independent components. 
It is well known that under mild conditions, the SDE converges exponentially fast to the target distribution $p(x)$ (see e.g., \citep{Markowich99onthe}). 
Since~\eqref{eqn:update_ujn} approximates the SDE \eqref{eqn:Langevin} with an $\mathcal{O}(h)$ discretization error, the probability distribution of $x^m$ produced by LMC~\eqref{eqn:update_ujn} converges exponentially to the target distribution up to a discretization error \citep{DALALYAN20195278}.

A significant drawback of LMC is that its dependence on the problem dimension $d$ is rather poor. In each iteration, the full gradient needs to be evaluated. However, in most practical problems, since the analytical expression of the gradient is not available, each partial derivative component in the gradient needs to be computed separately, either through finite differencing or automatic differentiation, so that the total number of such evaluations can be as many as  $d$ times the number of required iterations. 
In the weather prediction and epidemiology problems discussed above, $f$ stands for the map from the parameter space of measured quantities via the underlying partial differential equations (PDEs), and each dimensional partial derivative calls for one forward and one adjoint PDE solve. Thus, $2d$ PDE solves are required in general at each  iteration. Another example comes from the study of directed graphs with multiple nodes. Denote the nodes by ${\cal N} = \{1,2,\dotsc,d\}$ and directed edges by ${\cal E} \subset \{ (i,j) : i,j \in {\cal N} \}$, and suppose there is a scalar variable $x_i$ associated with each node. When the function $f$ has the form $f(x) = \sum_{(i,j) \in {\cal E}} f_{ij}(x_i,x_j)$, the partial derivative of $f$ with respect to $x_i$ is given by
\[
\frac{\partial f}{\partial x_i} = \sum_{j : (i,j) \in {\cal E}} \frac{\partial f_{ij}}{\partial x_i} (x_i,x_j) + \sum_{l : (l,i) \in {\cal E}} \frac{\partial f_{li}}{\partial x_i} (x_l,x_i)\,.
\]
Note that the number of terms in the summations equals the number of edges that touch node $i$, the expected value of which is about $2/d$ times the total number of edges in the graph. Meanwhile, evaluation of the full gradient would require evaluation of both partial derivatives of each $f_{ij}$ for {\em all} edges in the graph. Hence, the cost difference between these two operations is a factor of order $d$.

In this paper, we study how to modify the updating strategies of LMC to reduce the numerical cost, with the focus on reducing dependence on $d$. 
In particular, we will develop and analyze a method called Random Coordinate Langevin Monte Carlo (RC-LMC).
This idea is inspired by the random coordinate descent (RCD) algorithm from  optimization~\citep{doi:10.1137/100802001,Ste-2015}. RCD is a version of Gradient Descent (GD) in which one coordinate (or a block of coordinates) is selected at random for updating along its negative gradient direction. 
In optimization, RCD can be significantly cheaper than GD, especially when the objective function is skewed and the dimensionality of the problem is high. 
In RC-LMC, we use the same basic strategy: At iteration $m$, a single coordinate of $x^m$  is randomly selected for updating, while all others are left unchanged. 

% \sw{Here is the network / graph example.}

% \textcolor{purple}{An example of a function $f$ for which the expected cost of exact evaluation of one element of the gradient is of the order of $1/d$ of the cost of a full gradient is a function associated with a graph with nodes ${\cal N} = \{1,2,\dotsc,d\}$ and directed edges ${\cal E}$, where there is a scalar variable $x_i$ associated with each node $i=1,2,\dotsc,d$. The edge set ${\cal E}$ consists of ordered node pairs $(i,j)$, and the function $f$  has the form $f(x) = \sum_{(i,j) \in {\cal E}} f_{ij}(x_i,x_j)$. If we select a single index $i$, the partial derivative of $f$ w.r.t. $x_i$ is given by
% \[
% \frac{\partial f}{\partial x_i} = \sum_{j : (i,j) \in {\cal E}} \frac{\partial f_{ij}}{\partial x_i} (x_i,x_j) + \sum_{l : (l,i) \in {\cal E}} \frac{\partial f_{li}}{\partial x_i} (x_l,x_i).
% \]
% Note that the number of terms in the summations in this expression equals the number of edges in the graph that touch node $i$, the expected value of which is about $2/d$ times the total number of edges in the graph. Meanwhile, evaluation of the full gradient would require evaluation of both partial derivatives of each $f_{ij}$ for {\em all} edges in the graph. Hence, the cost difference between these two operations is a factor of order $d$.}

Although each iteration of RC-LMC is cheaper than conventional LMC, more iterations are required to achieve the target accuracy, and delicate analysis is required to obtain bounds on the total cost.
Analagous to  optimization, the savings of RC-LMC by comparison with LMC depends strongly on the structure of the dimensional Lipschitz constants. 
Under the assumption that there is a factor-of-$d$ difference in per-iteration costs, we conclude the following:
\begin{enumerate}[leftmargin=.5cm]
\item (Theorem~\ref{thm:rcolmc}) When the gradient of $f$ is Lipschitz but the Hessian is not, RC-LMC costs $\widetilde{O}(d^2/\epsilon^2)$ for an $\epsilon$-accurate solution, and it is cheaper than the classical LMC if $f$ is skewed and the dimension of the problem is high. 
The optimal numerical cost in this setting is achieved when the probability of choosing the $i$-th direction is proportional to the $i$-th directional Lipschitz constant.
\item (Theorem~\ref{thm:rcolmc2}) When both the gradient and the Hessian of $f$ are Lipschitz, RC-LMC requires $\widetilde{O}(d^{3/2}/\epsilon)$ iterations to achieve $\epsilon$ accuracy. On the other hand, the currently available result indicates that the classical LMC costs $\widetilde{O}(d^2/\epsilon)$. Thus, RC-LMC saves a factor of at least $d^{1/2}$. 
\item (Proposition~\ref{prop:badexampleW22}) The $\widetilde{O}(d^{3/2}/\epsilon)$ complexity bound for RC-LMC is sharp when both the gradient and the Hessian of $f$ are Lipschitz.
\end{enumerate}

The notation $\widetilde{O}(\cdot)$ omits the possible log terms. We make three additional remarks. 
(a) Throughout the paper we assume that one element of the gradient is available at an expected cost of approximately $1/d$ of the cost of the full gradient evaluation. 
Although this property is intuitive, and often holds in many situations (such as the graph-based example presented above), it does not hold for all problems~\citep{Ste-2015}.
(b) Besides replacing gradient evaluation by coordinate algorithms, one might also improve the dimension dependence of LMC by utilizing a more rapidly convergent method for the underlying SDEs than \eqref{eqn:Langevin}. One such possibility is to use underdamped Langevin dynamics, see e.g.,  \citep{doi:10.1063/1.436415, dalalyan2018sampling, Cheng2017UnderdampedLM,eberle2019,shen2019randomized,cao2019explicit}, which can also be combined with coordinate sampling. 
For the clarity of presentation, we will focus only on LMC in this work and leave the extension to underdampped samplers to a future work. 
(c) It is also possible to reduce the cost of full gradient evaluation using stochastic gradient \citep{welling2011bayesian}, but it requires a specific form of the objective function that is not considered in this work.%while we will focus on exact evaluation of gradients and partial derivatives in this work. 

The paper is organized as follows. We present the RC-LMC algorithm in Section~\ref{sec:alg}. Notations and assumptions on $f$ are listed in Section~\ref{sec:assumption}, where we also recall theoretical results for the classical LMC method.  
We present our main results regarding the numerical cost in Section~\ref{sec:results} and numerical experiments in Section~\ref{sec:numerics}. 
Proofs of the main results are deferred to  the Appendix.

\section{Random Coordinate Langevin Monte Carlo}\label{sec:alg}
We introduce the Random Coordinate Langevin Monte Carlo (RC-LMC) method in this section. 
At each iteration, one coordinate is chosen at random and updated, while the other components of $x$ are unchanged.
Specifically, denoting by $r^m$  the index of the random coordinate chosen at $m$-th iteration, we obtain $x^{m+1}_{r^m}$ according to a single-coordinate version of \eqref{eqn:update_ujn} and set $x^{m+1}_i=x^m_i$ for $i \ne r^m$.

The coordinate index $r^m$ can be chosen uniformly from $\{1,2,\dotsc,d\}$; but we will consider more general possibilities. Let  $\phi_i$ be the probability of component $i$ being chosen, we denote the distribution from which $r^m$ is drawn by $\Phi$, where
\begin{equation} \label{eq:def.Phi}
\Phi := \{\phi_1,\phi_2,\dotsc, \phi_d\}, \quad \mbox{where $\phi_i>0$ for all $i$ and $\sum_{i=1}^d \phi_i=1$.}
\end{equation}
The stepsize may depend on the choice of coordinate; we denote the stepsizes by $\{h_1,h_2,\dotsc,h_d\}$ and assume that they do not change across iterations. 
In this paper, we choose $h_i$ to be inversely dependent on probabilities $\phi_i$, as follows:
\begin{equation}\label{condition:pranh}
h_i=\frac{h}{\phi_i}\,, \quad i=1,2,\dotsc,d\,,
\end{equation}
where $h>0$ is a parameter that can be viewed as the expected stepsize. 
In Section~\ref{sec:case1_result}-\ref{sec:case2_result}, we will find the optimal form of $\Phi$ under different scenarios. The initial iterate $x^0$ is drawn from a distribution $q_0$, which  can be any distribution that is easy to draw from (the normal distribution, for example). 
We present the complete method in Algorithm~\ref{alg:RCD-OLMC}.

\begin{algorithm}[htb]
\caption{\textbf{Random Coordinate Langevin Monte Carlo} (RC-LMC)}\label{alg:RCD-OLMC}
\begin{algorithmic}
\State \textbf{Input:} Coordinate distribution $\Phi := \{\phi_1,\phi_2,\dotsc, \phi_d\}$; parameter $h>0$ and  stepsize set $\{h_1,h_2,\dotsc,h_d\}$ defined in \eqref{eq:def.Phi}--\eqref{condition:pranh}; $M$ (stop index).

\smallskip 
\State Sample 
$x^0$ from an initial distribution $q_0$
\For{$m=0,1,2,\dotsc M-1$}
\State 1. Draw $r^m \in \{1,\dots,d\}$ according to probability distribution $\Phi$;
\State 2. Draw $\xi^m$ from $\mathcal{N}(0,1)$;
\State 3. Update $x^{m+1}$ by 
\begin{equation}\label{alg:updatexm}
x^{m+1}_{i}=
\begin{cases}
x^m_{i}- h_{i} \partial_{i} f(x^m) +\sqrt{2h_{i}}\,\xi^{m}, & i = r^m \\
x^m_i, & i \neq r^m.
\end{cases}
\end{equation}
\EndFor
\State \Return $x^M$
\end{algorithmic}
\end{algorithm}

When we compare~\eqref{alg:updatexm} with the classical LMC~\eqref{eqn:update_ujn}, we see that in the updating formula, the gradient is replaced by a partial derivative in a random direction $r^m$:
\[
\nabla f(x^m) \to \partial_{r^m}f(x^m)\boldsymbol{e}_{r^m}\,,
\]
where $\boldsymbol{e}_i$ is the unit vector for $i$-th direction. Define the elapsed time at $m$-th iteration as
% As a consequence, per iteration, only one coordinate gets updated by the corresponding directional partial derivative.
\begin{equation}\label{palphasalpha}
T^m :=\sum^{m-1}_{n=0} h_{r^n}\,, \quad \text{and} \quad T^0 : = 0\,, 
\end{equation}
then for $t \in (T^m, T^{m+1}]$, the updating formula \eqref{alg:updatexm} can be viewed as the Euler approximation to the following SDE:
% we have
% \begin{equation}\label{eqn:xt}
% \left\{
% \begin{aligned}
% &x_{r^m}(t)=x_{r^m}({T^m})-\int^{t}_{T^m} \partial_{r^m} f(x^m)\rd s+\sqrt{2}\int^t_{T^m}\rd B_s\,, \\
% &x_{i}(t)=x_{i}({T^m}),\quad i\neq r^m\,,
% \end{aligned}
% \right.
% \end{equation}
% where we require $B_{T^{m+1}}-B_{T^m}=\sqrt{h_{r^m}}\xi^m$.
% This algorithm can be viewed as the Euler approximation to the following SDE:
\begin{equation}\label{eqn:LDSDE2continumm}
\left\{
\begin{aligned}
&X_{r^m}(t)=X_{r^m}(T^m)-\int^{t}_{T^m} \partial_{r^m} f(X(s))\rd s+\sqrt{2}\int^t_{T^m}\rd B_s\,,\\
&X_{i}(t)=X_i(T^m)\,,\quad \forall i\neq r^m\,.
\end{aligned}
\right.
\end{equation}
We note that the SDE preserves the invariant measure, that is, $X(t)\sim p$  for any $t\geq 0$. We discuss further the convergence property of the SDE \eqref{eqn:LDSDE2continumm} in Section~\ref{sec:crclmc}. 
% \sw{Does everyone know what "the invariant measure" is?}\jl{what is $q_2$?}~\zd{It's a typo, it's $p$, fixed.}

\section{Notations, assumptions and classical results}\label{sec:assumption}

We unify notations and assumptions in this section, and  summarize and discuss the classical results on LMC. 
Throughout the paper, to quantify the distance between two probability distributions, we use the Wasserstein distance defined by
\[
W(\mu,\nu) = \Bigl(\inf_{(X,Y)\in \Gamma(\mu,\nu)} \mathbb{E}|X -Y|^2\Bigr)^{1/2}\,,
\]
where $\Gamma(\mu,\nu)$ is the set of distribution of $(X,Y)\in\mathbb{R}^{2d}$ whose marginal distributions, for $X$ and $Y$ respectively, are $\mu$ and $\nu$. 
The distributions in $\Gamma(\mu,\nu)$  are called the {\em couplings} of $\mu$ and $\nu$. 
Due to the use of power $2$ in the definition, this is sometimes called the Wasserstein-$2$ distance.

We assume that $f$ is strongly convex, so that $p$ is strongly log-concave. 
We obtain results under two different assumptions: First, Lipschitz continuity of the gradient of $f$  (Assumption~\ref{assum:Cov}) and second, Lipschitz continuity of the Hessian of $f$  (Assumption~\ref{assum:Hessian} together with Assumption~\ref{assum:Cov}).

\begin{assum}\label{assum:Cov}
The function $f$ is twice differentiable, $f$ is $\mu$-strongly convex for some $\mu>0$ and its gradient $\nabla f$ is $L$-Lipschitz. That is, for all $x,x'\in\mathbb{R}^d$, we have
\begin{equation}\label{Convexity}
f(x)-f(x')-\nabla f(x')^\top (x-x')\geq \frac{\mu}{2} |x-x'|^2\,,
\end{equation}
and
\begin{equation}\label{GradientLip}
|\nabla f(x)-\nabla f(x')|\leq L|x-x'|\,.
\end{equation}
\end{assum}
It is an elementary consequence of \eqref{Convexity} that
\begin{equation} \label{Convexity.2}
(\nabla f(x') - \nabla f(x))^\top (x'-x) \ge \mu | x'-x|^2, \quad \mbox{for all $x,x' \in \mathbb{R}^d$.}
\end{equation}

Since each coordinate direction plays a distinct role in RC-LMC, we distinguish the Lipschitz constants in each such direction. 
When Assumption~\ref{assum:Cov} holds, partial derivatives in all coordinate directions are also Lipschitz. 
Denoting them as $L_i$ for each $i = 1, 2, \dotsc, d$, we have
\begin{equation}\label{GradientLipcoord}
|\partial_i f(x+t\boldsymbol{e}_i)-\partial_i f(x)|\leq L_i |t|
\end{equation}
for any $x\in\mathbb{R}^d$ and any $t \in \mathbb{R}$. We further denote $L_{\max}: =\max_i \, L_i$ and define  condition numbers as follows:
\begin{equation}\label{eqn:R}
\kappa=L/\mu\geq 1,\quad \kappa_i=L_i/\mu\geq1\,,\quad \kappa_{\max} = \max_i\kappa_i\,.
\end{equation}
As shown in ~\citep{Ste-2015}, we have
\begin{equation} \label{eq:LiL}
L_i \le L_{\max} \le L \le d L_{\max}, \quad \kappa_i \le \kappa_{\max} \le \kappa\le d\kappa_{\max}\,.
\end{equation}
These assumptions together imply that the spectrum of the Hessian is bounded above and below for all $x$, specifically, $\mu {I}_d\preceq \nabla^2 f(x) \preceq L{I}_d$ and $[\nabla^2 f(x)]_{ii} \le L_i\le L_{\max}$ for all $x \in \R^d$.

Both upper and lower bounds of $L$ in term of $L_{\max}$ in \eqref{eq:LiL} are tight. If $\nabla^2 f$ is a diagonal matrix, then $L_{\max}=L$, both being the biggest eigenvalue of $\nabla^2 f$. 
Thus,  $\kappa_{\max} = \kappa$ in this case. 
This is the case in which all coordinates are independent of each other, for example $f = \sum_i\lambda_ix_i^2$. On the other hand, if $\nabla^2f = \mathsf{e}\cdot\mathsf{e}^\top$
where $\mathsf{e}\in\mathbb{R}^d$ satisfies $\mathsf{e}_i=1$ for all $i$, then $L = dL_{\max}$ and $\kappa = d\kappa_{\max}$. 
This is a situation in which $f$ is highly skewed, that is,  $f = (\sum_i x_i)^2/2$.

The next assumption concerns higher regularity for $f$.
\begin{assum}\label{assum:Hessian}
The function $f$ is three times differentiable and $\nabla^2 f$ is H-Lipschitz, that is
\begin{equation}\label{HessisnLip}
\|\nabla^2 f(x)- \nabla^2 f(x')\|_2\leq H|x-x'|, \quad \mbox{for all $x,x'\in\mathbb{R}^d$}.
\end{equation}
\end{assum}

When this assumption holds, we further define $H_i$  to satisfy
\begin{equation}\label{HessisnLipcoord}
\lvert \partial_{ii}f(x+t\boldsymbol{e}_i)-\partial_{ii}f(x)\rvert \leq H_i |t|\,,
\end{equation}
for any $i=1,2,\dots,d$, all $x\in\mathbb{R}^d$, and all $t \in \mathbb{R}$, where $\partial_{ii} f$ is $[\nabla^2 f(x)]_{ii}$, the $(i,i)$ diagonal entry of the Hessian matrix $\nabla^2 f$.

We summarize existing results for the classical LMC in the following theorem.
\begin{theorem} [{\citep[Theorem~9]{durmus2018analysis},  \citep[Theorem~5]{DALALYAN20195278}}] \label{thm:convergenceolmc}
Let $q_m$ be the probability distribution of the $m$-th iteration of LMC~\eqref{eqn:update_ujn}, and $p$ be the target distribution. Using the notation $W_m := W(q_m,p)$, we have the following:
\begin{itemize}
\item Under Assumption~\ref{assum:Cov}, let $h\leq 1/L$,  we have
\begin{equation}\label{eqn:convergenceolmc}
W_m\leq \exp\left(-{\mu hm}/2\right)W_0+2(\kappa hd)^{1/2}\,;
\end{equation}
\item Under Assumptions~\ref{assum:Cov} and \ref{assum:Hessian}, let $h<2/(\mu+L)$, we have
\begin{equation}\label{eqn:convergenceolmcdalaya2}
W_m\leq \exp\left(-\mu hm\right)W_0+\frac{Hhd}{2\mu}+3\kappa^{3/2}\mu^{1/2}hd^{1/2}\,.
\end{equation}
\end{itemize}
\end{theorem}

This theorem yields stopping criteria for the number of iterations $M$ to achieve a user-defined accuracy of $\epsilon$.
When the gradient of $f$ is Lipschitz, to achieve $\epsilon$-accuracy, we can require both terms on the right hand side of~\eqref{eqn:convergenceolmc} to be smaller than $\epsilon/2$, which occurs when 
\begin{equation} \label{eqn:hm_classical_1}
h=\Theta( \epsilon^2/d\kappa)\,, \quad M=\Theta\left(\frac{1}{\mu h}\log\left(\frac{W_0}{\epsilon}\right)\right)=\Theta\left(\frac{d\kappa}{\mu\epsilon^2}\log\left(\frac{W_0}{\epsilon}\right)\right),
\end{equation}
leading to a cost of $\widetilde{O}(d^2\kappa/(\mu\epsilon^2))$ evaluations of gradient components (when we assume that each full gradient can be obtained at the cost of $d$ individual components of the gradient). 
When both the gradient and the Hessian are Lipschitz, to achieve $\epsilon$-accuracy, we require all three terms on the right hand side of~\eqref{eqn:convergenceolmcdalaya2} to be smaller than $\epsilon/3$. 
Assuming $d\gg 1$ and all other constants are $O(1)$, we thus obtain 
\begin{equation}\label{eqn:hm_classical_2}
h=\Theta( \epsilon\mu/dH)\,, \quad M=\Theta\left(\frac{dH}{\mu^2\epsilon}\log\left(\frac{W_0}{\epsilon}\right)\right)\,,
\end{equation}
which yields a cost of $\widetilde{O}(d^2H/(\mu^2\epsilon))$ evaluations of gradient components. Here $A = \Theta(B)$ denotes $cB\leq A\leq C B$ for some absolute constant $c$ and $C$.

\section{Main results}\label{sec:results}
%\jl{I feel that the main result appears too late for a 8 page paper; perhaps one can put an informal statement already in the introduction? }~\zd{Can we provide a table before? I am not sure if we have enough space.}
We discuss the main results from two perspectives. 
In Section~\ref{sec:crclmc} we examine the convergence of the underlying SDE~\eqref{eqn:LDSDE2continumm}, laying the foundation for the convergence in the discrete setting. 
We then build upon this result and show the convergence of the RC-LMC algorithm in Section~\ref{sec:case1_result} and~\ref{sec:case2_result} under two different assumptions. 
We show in Section~\ref{sec:tight} that when both Assumption~\ref{assum:Cov} and~\ref{assum:Hessian} are satisfied, our bound is tight with respect to $d$ and $\epsilon$.

\subsection{Convergence of the SDE~\eqref{eqn:LDSDE2continumm}}\label{sec:crclmc}

To study the convergence of~\eqref{eqn:LDSDE2continumm}, we first let $X^m=X(T^m)$ and denote the probability filtration by $\mathcal{F}^m=\left\{ x^0, r^{n\leq m}, B_{s\leq T^m}\right\}$.
Then $\left\{X^m\right\}^\infty_{m=0}$ is a Markov chain and the following theorem shows its geometric ergodicity.

\begin{theorem}\label{thm:contiummRCULMC}
Denote by $q_m(x)$ the probability density function of $X^m$. If $f$ satisfies Assumption \ref{assum:Cov} and $h\leq \frac{\mu \min\{\phi_i\}}{4+8L^2+32L^4}$, then $p(x)$ is the density of the stationary distribution of the Markov chain $\left\{X^m\right\}^\infty_{m=0}$. Furthermore, if the second moment of $q_0$ is finite and $X^0$ is drawn from $q_0$, then there are constants $R>0$ and $r>1$, independent of $m$, such that for any $m \geq 0$ we have
\begin{equation}\label{eqn:converge3}
\int_{\mathbb{R}^{d}} |q_m(x)-p(x)|\rd x\leq Rr^{-m}\,.
\end{equation}
\end{theorem}
% \begin{proof}
% See Appendix~\ref{sec:proofofthm:contiummRCULMC}.
% \end{proof}
See proof in Appendix~\ref{sec:proofofthm:contiummRCULMC}. This theorem states that the solution to the SDE converges to the target distribution. % and thus it lays the foundation for showing the convergence of the algorithm, the discretization of the SDE. 
Since the discrepancy between $q_m$ and $p$ decays exponentially in time on the continuous level, the discrete version (as computed in the algorithm) can be expected to converge as well. 
We will establish this fact in subsequent subsections. 
% \begin{remark}
% Unlike the results in~\cite{DALALYAN20195278} for the classical Langevin dynamics, we cannot trace the dependence of $R$ and $r$ on parameters such as $h$, $d$, $m$, and $L$. 
% The difficulty comes from the fact for the coordinate dynamics, the contraction property is no longer available. Thus, we turn to Lyapunov function tools \citep{MATTINGLY2002185} to establish the convergence in TV distance, and unfortunately these do not come with an explicit quantification of the rates. \end{remark}

\subsection{Convergence of RC-LMC. Case 1: Lipschitz gradient}\label{sec:case1_result}

Under Assumption~\ref{assum:Cov}, we have the following result. The proof can be found in Appendix~\ref{sec:appendixA}.
\begin{theorem}\label{thm:rcolmc} 
Assume $f$ satisfies Assumption~\ref{assum:Cov}, and $h_i=h/\phi_i$ with $h \leq \frac{\mu\min\left\{\phi_i\right\}}{8L^2}$.
% \begin{equation}\label{olmc:condition}
% h \leq \frac{\mu\min\left\{\phi_i\right\}}{8L^2}\,.
% \end{equation}
Let $q_m$ be the probability distribution of $x^m$ computed in~\eqref{alg:updatexm}, let $p$ be the target distribution, and denote $W_m := W(q_m,p)$. Then we have
\begin{equation}\label{eqn:convergencercolmc}
W_m\leq \exp\left(-\frac{\mu hm}{4}\right)W_0+\frac{5h^{1/2}}{\mu}\sqrt{\sum^d_{i=1}\frac{L^2_i}{\phi_i}}\,.
\end{equation}
\end{theorem}
%\begin{proof}
%See Appendix \ref{sec:appendixA}.
%\end{proof}

We make a few comments here: 
(1) the requirement on $h$ is rather weak. When both $\mu$ and $L$ are moderate (both $O(1)$ constants), the requirement is essentially $h\lesssim 1/d$. 
(2) The estimate \eqref{eqn:convergencercolmc} consists of two terms. The first is an exponentially decaying term and the second comes from the variance of random coordinate selection. If we assume all Lipschitz constants $L_i$ are of $O(1)$, this remainder term is roughly $O(h^{1/2}d)$. 
(3) The theorem suggests a stopping criterion: to have $W_M\leq\epsilon$, we roughly need $h<\epsilon^2/d^2$, and $M=\widetilde{O}({d^2}/{\epsilon^2})$, assuming $L_i=O(1)$. In terms of $\epsilon$ and $d$ dependence, this puts $M$ at the same order as~\eqref{eqn:hm_classical_1}, as required by the classical LMC. 
% \sw{Really? The cost for \eqref{eqn:hm_classical_1} seems to have an extra factor of $\kappa/\mu$. I guess you are assuming that this is $O(1)$, i.e. very nice conditioning, basically sampling a Gaussian.}\ql{I prefer to focus only on $d$ here, and leave the discussion on the comparison after the corollary. Is that okay?}
% \ql{I think the main point of this paper is to address the issue on $d$, and all other constants are considered fixed. If we want to optimize on conditioning, then the strategy has to change. We add a sentence in the intro saying we only care $d$ here.}~\zd{I think we can add one more comment to show the advantage of RC-LMC, we need to consider the case when $L_i$ is very small, which make $\sum_i L_i\ll dL$. If we assume $L_i\sim O(1)$, we can't see any advantage. I guess we shouldn't focus on $d$ in our paper, we should say it's always possible that $\sum_i L_i\ll dL$.}

Theorem~\ref{thm:rcolmc} holds for all choices of $\{\phi_i\}$ satisfying \eqref{eq:def.Phi}.
From the explicit formula~\eqref{eqn:convergencercolmc} we can choose $\{\phi_i\}$ to minimize the right-hand side of the bound. 
\citet{doi:10.1137/100802001}  proposed distributions $\Phi$ that depend on the  dimensional Lipschitz constants $L_i$, $i=1,2,\dotsc,d$ from \eqref{GradientLipcoord}.
For $\alpha \in\mathbb{R}$, we can let $\phi_i(\alpha)\propto L^\alpha_i$, specifically,
\begin{equation}\label{condition:prand}
\phi_i(\alpha):=\frac{L^\alpha_i}{\sum_jL^\alpha_j}\,,\quad\text{and}\quad \Phi(\alpha) := \{\phi_1(\alpha), \phi_2(\alpha), \dotsc, \phi_d(\alpha) \}\,.
\end{equation}
Note that when $\alpha=0$, $\phi_i(0)=1/d$ for all $i$: the uniform distribution among all coordinates. When $\alpha>0$, the directions that with larger Lipschitz constants have higher probability to be chosen. 
Since $h_i = {h}/{\phi_i}$, one uses smaller stepsizes for stiffer directions. 
(On the other hand, when $\alpha<0$, the directions with larger Lipschitz constants are less likely to be chosen, and the stepsizes are larger in stiffer directions, a situation that is not favorable and should be avoided.)
% Given $\alpha$, we can define the quantities
% \begin{equation}\label{eqn:K}
% K_\alpha:=\sum^d_{i=1} \kappa^\alpha_i\quad\text{so that} \;\; K_0 = d \;\;\text{and}\;\; K_1 = \sum_{i=1}^d\kappa_i\,.
% \end{equation}
% Note that when the condition number $\kappa_i$ is $O(1)$ for each $i$, then $K_1$ is a quantity of size $O(d)$. 
The following corollary discusses various choices of $\alpha$ and the corresponding computational cost.

\begin{cor}\label{cor:first}
Under the same conditions as in Theorem~\ref{thm:rcolmc}, with $\phi_i=\phi_i(\alpha)$ defined in~\eqref{condition:prand}, the number of iterations $M$ required to attain $W_M\leq \epsilon$ is $M= \Theta\left(\frac{K_{2-\alpha}K_\alpha}{\mu\epsilon^2}\log\left(\frac{W_0}{\epsilon}\right)\right)$, where $K_\alpha = \sum^d_{i=1} \kappa^\alpha_i$. This cost is optimized when $\alpha = 1$, for which we have
\begin{equation}\label{eqn:RC_LMC_1_m}
M = \Theta\left(\frac{(\sum_i\kappa_i)^2}{\mu\epsilon^2}\log\left(\frac{W_0}{\epsilon}\right)\right)\,. 
\end{equation}
% In particular:% \jl{I think we shall just state the optimal choice and results} \sw{Need to say here that you are talking specifically about distributions \eqref{condition:prand}.}
% \begin{itemize}
% \item[1.] When $\alpha = 0$, $\phi_i = {1}/{d}$, or $\alpha = 2$, $\phi_i\propto L_i^2$, the total cost is $\widetilde{O}\left(d(\sum_i\kappa_i^2)/(\mu\epsilon^2)\right)$;
% \item[2.] When $\alpha = 1$, $\phi_i\propto L_i$, the total cost is $\widetilde{O}\left((\sum_i\kappa_i)^2/(\mu\epsilon^2)\right)$; \jl{for the actual result, it would be good to state clearly the log factors}
% \item[3.] The optimal complexity estimate is achieved when $\alpha = 1$.
% \end{itemize}
\end{cor}
% \begin{proof}
% See Appendix~\ref{sec:appendixA}.
% \end{proof}
% \begin{proof}
% To ensure $W_m\leq\epsilon$, we set the two terms on the right hand side of~\eqref{eqn:convergencercolmc} to be smaller than $\epsilon/2$, which implies that
% \begin{equation}\label{eqn:olmccomputationalcomplexity}
% h=O\left(\frac{\mu^2\epsilon^2}{100\sum^d_{i=1}\frac{L^2_i}{\phi_i(\alpha)}}\right)\,,\quad \text{and} \quad m\geq \frac{4}{\mu h}\log\left(\frac{2W_0}{\epsilon}\right)\,.
% \end{equation}
% Using the definition of  $\phi_i(\alpha)$ according to \eqref{condition:prand}, we have 
% \[
% \sum^d_{i=1}\frac{L^2_i}{\phi_i(\alpha)}= \left(\sum_{i=1}^d \frac{L_i^{2}}{L_i^{\alpha}}\right)\left(\sum_{j=1}^dL_j^\alpha\right) = \mu^2K_{2-\alpha}K_\alpha\,,
% \]
% and thus $m= \widetilde{O}\left(\left(K_{2-\alpha}K_\alpha\right)/(\mu\epsilon^2)\right)$. Furthermore, $\alpha =1$ gives the optimal cost because:
% \[
% K_{2-\alpha}K_{\alpha} = \left(\sum\kappa_i^\alpha\right)\left(\sum\kappa_i^{2-\alpha}\right)\geq \left(\sum_i\kappa_i\right)^2 = K_1^2\,,
% \]
% due to the H\"older's inequality.
% \end{proof}
See proof in Appendix~\ref{sec:appendixA}. We note that the initial error $W_0$ enters through a $\log$ term and is essentially negligible.  To compare RC-LMC with the classical LMC, we compare~\eqref{eqn:RC_LMC_1_m} with~\eqref{eqn:hm_classical_1}, adjusting  \eqref{eqn:hm_classical_1} by a factor of $d$ to account for the higher cost per iteration. 
RC-LMC has more favorable computational cost if  $d^2\kappa \geq \left(\sum_i\kappa_i\right)^2$. 
Since $\kappa_i \leq \kappa_{\max}$, this is guaranteed if $\kappa\geq \kappa^2_{\max}$, which in turn is true when $\kappa\sim d\kappa_{\max}$ and $d>\kappa_{\max}$, that is, for highly skewed $f$ in high dimensional space. 
%When $f$ is highly skewed, e.g. $\kappa_1=\kappa_{\max}\gg \kappa_i$ for all $i\geq 2$, this inequality holds true when $d>\kappa_{\max}$. This means when the problem is in high dimensional space, and that $f$ is highly skewed, RC-LMC is cheaper than LMC.\ql{more discussion needed}

% Suppose we loosen $\kappa_i$ up to $\kappa_{\max}$, then the cost estimate for different values of $\alpha$ is the same: $\widetilde{O}(d^2\kappa^2_{\max}/\mu\epsilon^2)$. To compare this with $\widetilde{O}(d^2\kappa/\mu\epsilon^2)$ for the classical LMC, it is not immediate which one is smaller. To have RC-LMC being cheaper, we need
% \[
% \kappa_{\max}^2\leq \kappa\,.
% \]
% Noting~\eqref{eq:LiL}, this is the case when $f$ is highly skewed (having $\kappa\approx d\kappa_{\max}$) and the dimensionality is high $d>\kappa_{\max}$.~\zd{I am not sure what this sentence means.}

Our proof of Theorem~\ref{thm:rcolmc} follows from a coupling approach similar to that used by~\citet{DALALYAN20195278} for LMC. 
We emphasize that for the coordinate algorithm, we need to overcome the additional difficulty that the process of each coordinate is not contracting on the SDE~\eqref{eqn:LDSDE2continumm} level. This is a different situation from the classical LMC~\citep{DALALYAN20195278} whose corresponding SDE~\eqref{eqn:Langevin} already provides the contraction property and thus only the discretization error needs to be considered. Despite this, the algorithm RC-LMC still enjoys the contraction property that ensures that the distance between two different trajectories following the algorithm contract. However, this contraction property is not component-wise, so we need to choose Young's constant wisely and take summation of every coordinate. The summation will also produce some extra terms, which we need to bound. 
% \sw{I reworded this but not sure I got it right. Can you give a more specific pointer to where in the appendix this "fake contraction" property is discussed? I may be able to track it down once I get to the appendix.}~\zd{There was a typo I made before. We didn't construct a "fake contraction" for SDE. I try to revise the statement before. We use contraction property for algorithm since this way is much simpler. We didn't discuss this clearly in the proof. But it can be seen in \eqref{eqn:Deltam+1new1}. In this equation, if we don't use RCD, the first two terms will be $(1+a)(|\Delta|^2-2h\left\langle \Delta,f(x)-f(\widetilde{x})\right\rangle)$, which can directly give us contraction. However, in our case, we need to adjust $a$ for each $i$ and sum them to get the contraction. Because we choose different $a$ for different $i$, this contraction will have extra term, which is third term in \eqref{eqn:Deltam+1new1nonuni2}. Fortunately, we can still bound it.} 
\citet{DALALYAN20195278} obtains an estimate for the cost of the classical LMC of $\widetilde{O}(d^2\kappa^2/(\mu\epsilon^2))$. Compared with this estimate, our estimate for the cost of RC-LMC is always cheaper (since $\kappa^2\geq\kappa^2_{\max}$). The improved estimate of the cost of LMC \eqref{eqn:hm_classical_1} was obtained by \citet{durmus2018analysis} using a quite different approach based on optimal transportation. It is not clear whether their technique can be adapted to the coordinate setting to obtain an improved estimate.

\subsection{Convergence of RC-LMC. Case 2: Lipschitz Hessian}\label{sec:case2_result}
We now assume that Assumption~\ref{assum:Cov} and~\ref{assum:Hessian} hold, that is, both the gradient and the Hessian of $f$ are Lipschitz continuous. 
In this setting, we obtain the following improved convergence estimate. The proof can be found in Appendix~\ref{sec:appendixB}.
\begin{theorem}\label{thm:rcolmc2}
Assume $f$ satisfies Assumptions~\ref{assum:Cov} and \ref{assum:Hessian} and let $h_i=h/\phi_i$, with $h\leq\frac{\mu\min\left\{\phi_i\right\}}{8L^2}$.
% \begin{equation}\label{olmc:condition2}
% h\leq\frac{\mu\min\left\{\phi_i\right\}}{8L^2}\,.
% \end{equation}
Denoting by $q_m(x)$ the probability density function of $x^m$ computed from~\eqref{alg:updatexm} and by $p$ the target distribution, and letting $W_m := W(q_m,p)$, we have:
\begin{equation}\label{eqn:convergencercolmc2}
W_m\leq \exp\left(-\frac{\mu hm}{4}\right)W_0+\frac{3h}{\mu}\sqrt{\sum^d_{i=1}\frac{\left(L^3_i+H^2_i\right)}{\phi_i^2}}\,.
\end{equation}
\end{theorem}
%\begin{proof}
%See Appendix \ref{sec:appendixB}.
%\end{proof}

We see again two terms in the bound, an exponentially decaying term and a variance term. Assuming all Lipschitz constants are $O(1)$, the variance term is of $O(hd^{3/2})$. 
By comparing with Theorem~\ref{thm:rcolmc}, we see that $\epsilon$ error can be achieved with the looser stepsize requirement $h\lesssim\frac{\epsilon}{d^{3/2}}$.

By choosing $\{\phi_i\}$ to optimize the bound in Theorem~\ref{thm:rcolmc2}, we obtain the following corollary.
% We also have an immediate corollary discussing the numerical cost: \jl{with the result of the Theorem, shouldn't we choose $\phi_i$ proportional to some power of $L_i^3 + H_i^2$ instead of just using $L_i$?}~\zd{This is a very good idea, the optimal choice is 
% \[
% \phi_i=\frac{\left(L_i^3 + H_i^2\right)^{1/3}}{\sum^d_{i=1}\left(L_i^3 + H_i^2\right)^{1/3}}
% \]
% by using Lagrange multiplier. Then the final cost will be 
% \[
% \widetilde{O}\left(\frac{\left(\sum^d_{i=1}\left(L_i^3 + H_i^2\right)^{1/3}\right)\left(\sum^d_{i=1}\left(L_i^3 + H_i^2\right)^{2/3}\right)^{1/2}}{\mu^2\epsilon}\right)
% \]
% I think we should mention this optimal choice in the paper.
% }
% 
\begin{cor}\label{cor:second}
Under the same conditions as in Theorem~\ref{thm:rcolmc2}, the optimal choice of $\{\phi_i\}$ is to set:
\[
\phi_i=\frac{\left(L_i^3 + H_i^2\right)^{1/3}}{\sum^d_{i=1}\left(L_i^3 + H_i^2\right)^{1/3}}\,.
\]
For this choice, the number of iterations $M$ required to guarantee $W_M \le \epsilon$ satisfies
\begin{equation}\label{eqn:m_Hessian}
M=\Theta\left(\frac{\left(\sum^d_{i=1}\left(L_i^3 + H_i^2\right)^{1/3}\right)\left(\sum^d_{i=1}\left(L_i^3 + H_i^2\right)^{2/3}\right)^{1/2}}{\mu^2\epsilon}\log\left(\frac{W_0}{\epsilon}\right)\right).
\end{equation}
If $\mu$, $\kappa_i$ and $H_i$ are all constants of $O(1)$, then the total cost is $\widetilde{O}(d^{3/2}/\epsilon)$ regardless of the choice of $\{\phi_i\}$.
\end{cor}

This is a significant improvement compared to the cost of the classical LMC (which requires $\widetilde{O}(d^2/\epsilon)$ \citep{dalalyan2018sampling}), regardless of the structure of $f$. Indeed, the cost is reduced by a factor of $d^{1/2}$, which can be significant for high dimensional problems. 

\subsection{Tightness of the complexity bound}\label{sec:tight}
When both the gradient and the Hessian are Lipschitz, we claim that estimate $\widetilde{O}(d^{3/2}/\epsilon)$ obtained in Corollary~\ref{cor:second} is tight. 
An example is presented in the following proposition. %See proof in Appendix~\ref{sec:proofofthmbadexample}.
\begin{proposition}\label{prop:badexampleW22}
Let $\phi_i = 1/d$ for all $i$, and set the initial distribution and the target distribution to be:
\begin{equation} \label{eq:hs}
q_0(x)=\frac{1}{(4\pi)^{d/2}}\exp(-|x-\mathsf{e}|^2/4)\,,\quad p(x)=\frac{1}{(2\pi)^{d/2}}\exp(-|x|^2/2)\,,
\end{equation}
where $\mathsf{e}\in\mathbb{R}^d$ satisfies $\mathsf{e}_i=1$ for all $i$. Let $q_m$ be the probability distribution of $x^m$ generated by Algorithm~\ref{alg:RCD-OLMC}, and denote $W_m := W(q_m,p)$. 
Then we have
\begin{equation}\label{eqn:badexampleW2bound2}
W_m\geq \exp\left(-2mh\right)\frac{\sqrt{d}}{3}+\frac{d^{3/2}h}{6}\,,\quad m\geq 1\,.
\end{equation}
In particular, to have $W_M\leq \epsilon$, one needs at least $M=\widetilde{O}(d^{3/2}/\epsilon)$.
\end{proposition}
% \begin{proof}
% See Appendix~\ref{sec:proofofthmbadexample}.
% \end{proof}

See proof in Appendix~\ref{sec:proofofthmbadexample}.
% For the given initial data, we have $W_0 = \sqrt{4-2\sqrt{2}}\, d^{1/2}$. 
% To have $W_m\leq \epsilon$, both terms in~\eqref{eqn:badexampleW2bound2} must be smaller than $\epsilon$. Thus we have $h\lesssim \epsilon/d^{3/2}$, 
% $m\sim \frac{1}{h}\log(\epsilon/\sqrt{d})$, which implies a cost of $\widetilde{O}(d^{3/2}/\epsilon)$, as claimed.
% % This demonstrates tightness of the bound of Theorem~\ref{thm:rcolmc2}. 

\section{Numerical results}\label{sec:numerics}
We provide some numerical results in this section. 
Since it is extremely challenging to estimate the Wasserstein distance between two distributions in high dimensions, we demonstrate instead the convergence of estimated expectation for a given observable.  Denoting by $\{x^{(i),M}\}_{i=1}^N$ the list of $N$ samples, with each of them computed through Algorithm~\ref{alg:RCD-OLMC} independently with $M$ iterations, we define  the error as follows:
\begin{equation}\label{MSEerror}
\textrm{Error}_M=\left|\frac{1}{N}\sum^N_{i=1}\psi(x^{(i),M})-\EE_p(\psi)\right|\,,
\end{equation}
where $\psi$ is a test function and $\EE_p(\psi)$ is the expectation of $\psi$ under the target distribution $p$.
As $h\rightarrow0$ and $Mh\to\infty$, we have $W_M\to0$, and $x^{(i),M}$ can be regarded as approximately sampled from $p$.
Thus, according to the central limit theorem, we have $\textrm{Error}_M = O(1/\sqrt{N})$.

% In Example 1, we set the target to be a Gaussian distribution in $d=100$ dimensional space:
% \[
% p(x)\propto\exp\left(-50|x_1|^2-\sum^{100}_{i=2}\frac{|x_i|^2}{2}\right),\quad q_0(x)\propto \exp\left(-\sum^{100}_{i=1}\frac{|x_i-1|^2}{2}\right)
% \]
% We run the simulation with $N=10^5$ particles and we measure the mean-square error (MSE) with the test function $\psi(x)=100|x_1|^2+\sum^{10}_{i=2}|x_i|^2$. The RC-LMC is ran with $\Phi$ chosen according to \eqref{condition:prand} with $\alpha=0$ and $1$. In the left panel of Figure \ref{Figure3} we see that RC-LMC with $\alpha = 0$ performs as well as the classical LMC, and RC-LMC with $\alpha = 1$ provides smaller $\textrm{MSE}$ at the same cost.
% % \begin{figure}[htbp]
% %      \centering
% %       \includegraphics[height = 0.15\textheight, width = 0.5\textwidth]{RCDOLMCtime2.eps}
% %      \caption{The decay of MSE with respect to the number of finite difference approximations.\ql{unify the labels of this plot and the next}}
% %      \label{Figure1}
% % \end{figure}

In this example, we set the target and initial distributions to be Gaussian $p(x)\propto p_1(\mathsf{x})p_2(x)$ and $q_0(x) \propto p_1(\mathsf{x}-\mathsf{e})p_2(x)$ with
% \[
% p(x)\propto\exp\left(-\frac{1}{2}\mathsf{x}\left(\mathsf{T}+\frac{d}{10}\mathsf{I}\right)^\top\left(\mathsf{T}+\frac{d}{10}\mathsf{I}\right)\mathsf{x}^\top-\sum^{100}_{i=11}\frac{|x_i|^2}{2}\right)
% \]
\[
p_1(\mathsf{x}) = \exp\left(-\frac12\mathsf{x}\left(\mathsf{T}+(d/10)I\right)^\top\left(\mathsf{T}+(d/10)I\right)\mathsf{x}^\top\right)\,,\quad p_2= \exp\left(-\frac12 \sum^{100}_{i=11} |x_i|^2\right)\,,
\]
% \[
% p(x)\propto\exp\left(-\frac12\mathsf{x}\left(\mathsf{T}+(d/10)\mathsf{I}\right)^\top\left(\mathsf{T}+(d/10)\mathsf{I}\right)\mathsf{x}^\top-\frac12 \sum^{100}_{i=11} |x_i|^2\right)
% \]
% \[
% q_0(x)\propto \exp\left(-\frac{1}{2}(\mathsf{x}-0.5\mathsf{e})\left(\mathsf{T}+\frac{d}{10}\mathsf{I}\right)^\top\left(\mathsf{T}+\frac{d}{10}\mathsf{I}\right)(\mathsf{x}-0.5\mathsf{e})^\top-\sum^{100}_{i=11}\frac{|x_i|^2}{2}\right)
% 
% \[
% q_0(x)\propto \exp\left(-\frac{1}{2}(\mathsf{x}-(\mathsf{e}/2))^\top\left(\mathsf{T}+(d/10)\mathsf{I}\right)^\top\left(\mathsf{T}+(d/10)\mathsf{I}\right)(\mathsf{x}-(\mathsf{e}/2))-\frac12 \sum^{100}_{i=11} |x_i|^2\right)
% \]
where $\mathsf{x}=\left(x_1,x_2,\dots,x_{10}\right)^\top$, $\mathsf{e}=\left(1,1,\dots,1\right)^\top \in \R^{10}$, $I$ is the identity matrix and $\mathsf{T}$ is a random matrix with each entry i.i.d. drawn from $\mathcal{N}(0,1)$. We run the simulation with $N=10^6$, and we compute $\textrm{Error}_M$ with $\psi(x) = \|\mathsf{x}\mathsf{x}^\top\|_2$. This measures the spectral norm of the covariance matrix of the first $10$ entries. As shown in Figure \ref{Figure3}, RC-LMC with $\alpha=1$ converges faster than RC-LMC with $\alpha = 0$, and both converge faster than the classical LMC.
% we can see RC-LMC will behave better than LMC when $f$ is skewed. Furthermore, if we choose $\phi_i$ smartly ($\alpha=1$), we can see large improvement.
\begin{figure}[htbp]
     \centering
      \includegraphics[height = 0.15\textheight, width = 0.45\textwidth]{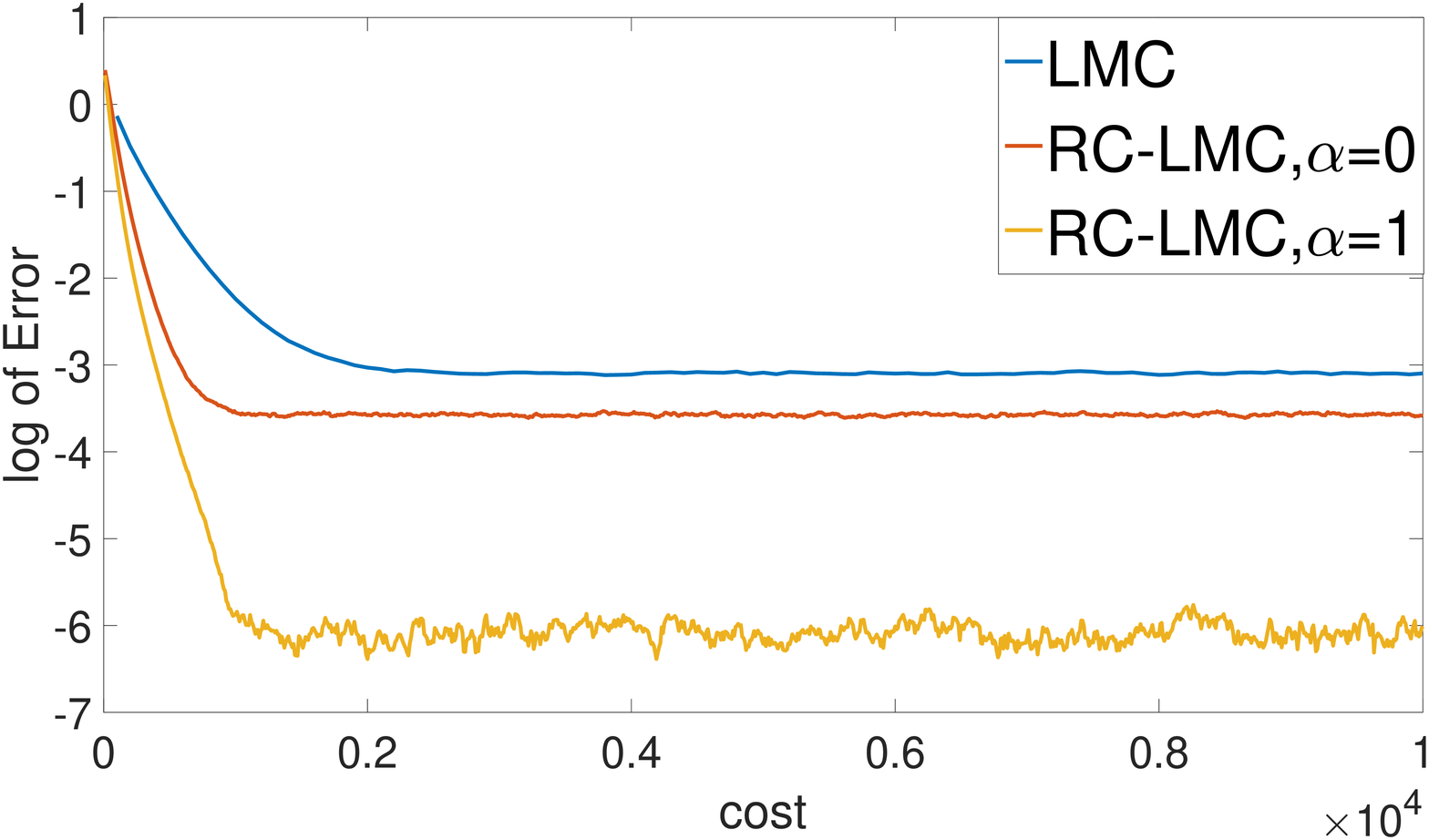}
     \caption{The decay of error with respect to the cost (number of $\partial f$ calculations). 
     % \jl{this is the first time finite difference approximation is mentioned? (also in the label of the figure)}}~\zd{I want to change it to cost and explain it in the text.}\ql{I am not sure I understand -- we mentioned a lot of times of ``the number of $\partial f$ calculations", right?}~\zd{Previously, we have a typo and the caption was number of finite difference approximation, I guess it's better to mention number of $\partial f$ calculations is the cost here.
     }
     \label{Figure3}
\end{figure}

\newpage
\bibliographystyle{apalike}
\bibliography{iclr2021_conference}

\newpage
\begin{appendix}
\section{Proof of Theorem \ref{thm:contiummRCULMC}}\label{sec:proofofthm:contiummRCULMC}
We recall the SDE \eqref{eqn:LDSDE2continumm}:
\begin{equation}\label{eqn:LDSDE2continumm2}
\left\{
\begin{aligned}
&X_{r^m}(t)=X_{r^m}(T^m)-\int^{t}_{T^m} \partial_{r^m} f(X(s))\rd s+\sqrt{2}\int^t_{T^m}\rd B_s\,,\\
&X_{i}(t)=X_i(T^m)\,,\quad \forall i\neq r^m\,,
\end{aligned}
\right.
\end{equation}
where $r^m$ is randomly selected from $1,\dotsc, d$. Moreover, recall that  $X^{m+1}=X\left(T^{m+1}\right)$ is a Markovian process. We denote its transition kernel by $\Xi$, meaning that 
\[
X^{m+1}\stackrel{d}{=} \Xi(X^m, \cdot)\,.
\]
Moreover, we denote $\Xi^n$ the $n$-step transition kernel. The following proposition establishes the exponential convergence of the Markov chain. 
\begin{proposition}\label{prop:crcolmc}
Under conditions of Theorem \ref{thm:contiummRCULMC}, there are constants $R_1>0,r_1>1$, such that for any $x^0\in\mathbb{R}^{d}$
\begin{equation}\label{eqn:converge2}
\sup_{A\in\mathcal{B}(\mathbb{R}^{d})}\left|\Xi^{md}(x^0,A)-\int_Ap(x)\rd x\right|\leq \left(|x^0-x^*|^2+1\right)R_1r^{-m}_1\,,
\end{equation}
where $x^\ast$ is the minimal point of $f(x)$ and $\Xi$ is the transition kernel for $\left\{X^m\right\}^\infty_{m=0}$. 
\end{proposition}

We postpone the proof of Proposition \ref{prop:crcolmc} to Section \ref{proofofeqnconverge2}. Now, we are ready to prove the theorem.

\begin{proof}[Proof of Theorem~\ref{thm:contiummRCULMC}]
First, suppose the distribution of $X^m$ is induced by $p$. 
Then for $i\neq r^m$, the distribution of $X_i(t)$ between $[T^m,T^{m+1}]$ is preserved. 
Meanwhile, we have
\[
\rd X_{r^m} = -\partial_{r^m}f(X(s))\rd t + \sqrt{2}\rd B_s\,,
\]
and the marginal distribution of $X_{r^m}(t)$ is also preserved. 
Therefore $X^{m+1}\sim p$, proving that $p(x)$ is the density of the stationary distribution.

Second, to prove \eqref{eqn:converge3}, let $x^0\sim q_0$ that has finite second moment, we multiply $q_0$ on both sides of~\eqref{eqn:converge2} and integrate, to obtain
\[
\int_{\mathbb{R}^{d}} |q_{md}(x)-p(x)|\rd x\leq C_0r^{-m}_1\,,
\]
where $C_0$ is a constant. 

By using~\eqref{eqn:LDSDE2continumm2} with It\^o's formula, we have
\[
\begin{aligned}
\frac{\rd\EE|X_{r^m}(t)|^2}{\rd t}&= -2\EE\left(\partial_{r^m}f(X_{r^m}(t))X_{r^m}(t)\right)+2\leq 2+\EE|\partial_{r^m}f(X_{r^m}(t))|^2+\EE|X_{r^m}(t)|^2\\
&\leq 2+L^2_{r^m}\EE|X_{r^m}(t)-x^*_{r^m}|^2+\EE|X_{r^m}(t)|^2\leq C_{1,r^m}\EE|X_{r^m}(t)|^2+C_{2,r^m}\,,
\end{aligned}
\]
where $C_{1,r^m}$ and $C_{2,r^m}$ are constants that depend only on $x^*$ and $L_{r^m}$. 
From Gr\"onwall's inequality,  we obtain 
\[
\EE\left(|X^{m+1}_i|^2\middle|r^m=i\right)\leq \exp(C_{1,i}h_i)\left[\EE(|X^{m}_i|^2)+C_{2,i}h_i\right]\,, \quad \mbox{for all $i=1,2,\dotsc,d$.}
\]
Then, if $\EE|X^{m}|^2<\infty$, we have for any $i=1,2,\dotsc,d$ that 
\[
\begin{aligned}
\EE\left(|X^{m+1}_i|^2\right)&=\frac{1}{d}\EE\left(|X^{m+1}_i|^2\middle|r^m=i\right)+\left(1-\frac{1}{d}\right)\EE\left(|X^{m+1}_i|^2\middle|r^m\neq i\right)\\
&\leq \frac{1}{d}\exp(C_{1,i}h_i)\left[\EE(|X^{m}_i|^2)+C_{2,i}h_i\right]+\left(1-\frac{1}{d}\right)\EE(|X^{m}_i|^2)<\infty\,,
\end{aligned}
\]
which implies $\EE|X^{m+1}|^2<\infty$. 
Therefore, if $q_0$ has finite second moment, then $q_i$ all have finite second moments for $i = 1, \dotsc, d-1$. 
%\sw{So here and below, it should be $d-1$ and not $d$?}~\zd{Actually, this is true for all $i$, but we only need $i\leq d-1$ since $q_{md+d}$ is same as $q_{(m+1)d}$, and this can be covered before.} 
Letting $x^0\sim q_i$, multiplying $q_i$ on both sides of \eqref{eqn:converge2} and integrating, we obtain
\[
\int_{\mathbb{R}^{d}} |q_{md+i}(x)-p(x)|\rd x\leq C_ir^{-m}_1\,,
\]
where $C_i$ is a constant. 
Since this bound holds for all $0\leq i\leq d-1$, we set $R=(\max_i C_i)r_1$ and $r=r^{1/d}_1$ to obtain~\eqref{eqn:converge3}.
\end{proof}

\subsection{Proof of Proposition \ref{prop:crcolmc}}\label{proofofeqnconverge2}
Before we prove the Proposition, we first recall a result from \citep{MATTINGLY2002185} for the convergence of Markov chain using Lyapunov condition together with minorization condition.
\begin{theorem}\label{thm:Mat}[\citep[Theorem~2.5]{MATTINGLY2002185}]
Let $\{X^n\}^{\infty}_{n=0}$ denote the Markov chain on $\mathbb{R}^d$ with transition kernel $\Xi$ and filtration $\mathcal{F}^n$. Let $\{X^n\}^{\infty}_{n=0}$ satisfy the following two conditions:
\begin{enumerate}[wide, labelwidth=!, labelindent=0pt]
\item[Lyapunov condition:] There is a function $L:\mathbb{R}^d\rightarrow [1,\infty)$, with $\lim_{x\rightarrow\infty} L(x)=\infty$, and real numbers $\alpha\in(0,1)$, and $\beta\in [0,\infty)$ such that
\[
\EE\left(L(X^{n+1})\middle|\mathcal{F}^n\right)\leq \alpha L(X^{n})+\beta\,.
\]
\item[Minorization condition:] For $L$ from the Lyqpunov condition, define the set $C\subset \mathbb{R}^d$  as follows:
\begin{equation} \label{eq:def.C}
C=\left\{ x\in\mathbb{R}^d \mid L(x)\leq \frac{2\beta}{\gamma-\alpha}\right\}\,,
\end{equation}
for some $\gamma\in(\alpha^{1/2},1)$. Then there exists an $\eta>0$ and a probability measure $\mathcal{M}$ supported on $C$ (that is,  $\mathcal{M}(C)=1$), such that
\[
\Xi(x,A)\geq \eta \mathcal{M}(A),\quad \forall A\in\mathcal{B}(\mathbb{R}^{d}), \; x\in C\,.
\]
\end{enumerate}
Under these conditions, the Markov chain $\{X^n\}^\infty_{n=0}$ has a unique invariant measure $\pi$. 
Furthermore, there are constants $r\in(0,1)$ and $R\in(0,\infty)$ such that, for any $x_0\in\mathbb{R}^d$, we have
\begin{equation}\label{eqn:mat}
\sup_{A\in\mathcal{B}(\mathbb{R}^{d})}\left|\Xi^{n}(x^0,A)-\pi(A)\right|\leq L(x_0)Rr^{-n}\,.
\end{equation}
\end{theorem}

To use this result to prove Proposition~\ref{prop:crcolmc}, we will consider the $d$-step chain of $\{X^n\}$ and verify the two conditions, as in the following two lemmas for the Lyapunov function and the minorization over a small set, respectively. 

\begin{lemma}\label{lem:convergenceulmclemma} Assume $f$ satisfies Assumption \ref{assum:Cov} and 
\begin{equation}\label{thm4.1hdependence}
h\leq \frac{\mu \min\{\phi_i\}}{4+8L^2+32L^4}\,,
\end{equation}
where $L$ is the Lipschitz constant defined in \eqref{GradientLip}.
% \sw{So the "$L$" here is the Lipschitz constant defined at the start of the paper, right? Nothing to do with $L(x)$?}~\zd{Yes, I add one sentence to clarify it.}
Let the Lyapunov function be $L(x)=|x-x^*|^2+1$, then we have:
\begin{equation}\label{eqn:lyconditionlemma}
\EE\left(L(X^{m+1})\middle|\mathcal{F}^m\right)\leq \alpha_1 L(X^{m})+\beta_1
\end{equation}
with 
\[
\alpha_1=1-\mu h\,,\quad\beta_1=(24+120L^2+\mu)h\,.
\]
\end{lemma}
\begin{lemma}\label{lemma:converge2}
Under conditions of Lemma \ref{lem:convergenceulmclemma}, with $L(x)=|x-x^*|^2+1$, let $\Xi$ denote the transition kernel. Define the set $C \subset \R^d$ as in \eqref{eq:def.C}, for some $\gamma\in(\alpha^{1/2},1)$.
Then there exists an $\eta>0$ and a probability measure $\mathcal{M}$ with $\mathcal{M}(C)=1$, such that
\begin{equation}\label{minorizationcondition}
\Xi^d(x,A)\geq \eta \mathcal{M}(A),\quad \forall A\in\mathcal{B}(\mathbb{R}^{d}),x\in C\,.
\end{equation}
\end{lemma}

Proposition~\ref{prop:crcolmc} follows easily from these results.
\begin{proof}[Proof of Proposition \ref{prop:crcolmc}]
It suffices to show $d$-step chain $\bigl\{X^{md}\bigr\}^\infty_{m=0}$ satisfies the conditions in Theorem~\ref{thm:Mat} with $L(x) = |x-x^\ast|^2+1$, $\alpha = \alpha_1^d$ and $\beta=d\beta_1$, and $\pi$ is induced by $p$. 
We apply \eqref{eqn:lyconditionlemma} from Lemma \ref{lem:convergenceulmclemma}  iteratively, $d$ times, to obtain
\[
\EE\left(L\left(X^{(m+1)d}\right)\middle|\mathcal{F}^{md}\right)\leq \alpha_1^d L\left(X^{md}\right)+d\beta_1\,,
\]
% \sw{Is this right? On the right-hand side, I get something more like $\alpha_1^d L(X^{md}) + \beta_1/(1-\alpha_1) =\alpha_1^d L(X^{md}) + \beta_1/(\mu h)$}~\zd{This is right, we just enlarge $\alpha^{i}$ to $1$ here. See the following calculation:
% \[
% \EE\left(L\left(X^{(m+1)d}\right)\middle|\mathcal{F}^{md}\right)\leq \alpha_1^d L\left(X^{md}\right)+\sum^{d-1}_{i=0} \alpha^i\beta_1\leq \alpha_1^d L\left(X^{md}\right)+\sum^{d-1}_{i=0} \beta_1=\alpha_1^d L\left(X^{md}\right)+d \beta_1
% \]}
which implies that $\left\{X^{md}\right\}^\infty_{m=0}$ satisfies Lyapunov condition in Theorem~\ref{thm:Mat} with $\alpha = \alpha_1^d$. 
Moreover, Lemma~\ref{lemma:converge2} directly implies that the $d$-step transition kernel satisfies the minorization condition. 
Therefore, by Theorem~\ref{thm:Mat}, we have
\begin{equation*}
\sup_{A\in\mathcal{B}(\mathbb{R}^{d})}\left|\Xi^{md}(x^0,A)-\pi(A)\right|\leq L(x_0)Rr^{-m}\,,
\end{equation*}
which concludes the proof of the proposition when we substitute $\pi(A) = \int_Ap(x)\rd{x}$.
\end{proof}

\begin{proof}[Proof of Lemma \ref{lem:convergenceulmclemma}]
We assume without loss of generality that $x^*=0 \in \R^d$ (so that $L(x)=|x|^2+1$) and drop the filtration $\mathcal{F}^m$ in the formula for simplicity of notation. Then
\begin{equation}\label{eqn:sumup}
\EE\left(L\left(X^{m+1}\right)\right)=\sum^d_{i=1}\phi_i\EE\left(L\left(X^{m+1}\right)\middle|r^m=i\right)\,.
\end{equation}
Since 
\[
\begin{aligned}
L\left(X^{m+1}\right) &= |X^{m+1}|^2+1 = |X^m + (X^{m+1}-X^m)|^2+1 \\
&= L\left(X^m\right) + 2X^m(X^{m+1}-X^m) + |X^{m+1}-X^m|^2\,,
\end{aligned}
\]
we have
\begin{equation}\label{eqn:taylor}
\begin{aligned}
\EE\left(L\left(X^{m+1}\right)\middle|r^m=i\right)=&L\left(X^m\right)+2\EE\left[X^m_i\left(X^{m+1}_{i}-X^m_{i}\right)\middle|r^m=i\right]\\
&+\EE\left[\left(X^{m+1}_{i}-X^m_{i}\right)^2\middle|r^m=i\right]\,.
\end{aligned}
\end{equation}
To deal with second term and third term in \eqref{eqn:taylor}, we first note that, under condition $r^m=i$:
\begin{equation}\label{eqn:equationsxsw}
X^{m+1}_{i}-X^m_{i}=-\int^{T^m+h_{i}}_{T^m}\partial_{i} f(X(s))\rd s+\sqrt{2}\int^{T^m+h_{i}}_{T^m}\rd B_s\,.
\end{equation}
This means 
\begin{equation}\label{partialxiLxi}
\begin{aligned}
&2\EE\left[X^m_i\left(X^{m+1}_{i}-X^m_{i}\right)\middle|r^m=i\right]\\
=& -2\EE\left[X^m_i\int^{T^m+h_{i}}_{T^m}\partial_{i} f(X(s))\rd s\middle|r^m=i\right]\\
=& -2h_iX^m_i\partial_if(X^m)-2\EE\left[X^m_i\int^{T^m+h_{i}}_{T^m}\left(\partial_{i} f(X(s))-\partial_if(X^m)\right)\rd s\middle|r^m=i\right]\,.
\end{aligned}
\end{equation}
We further bound the second term of \eqref{partialxiLxi}:
\begin{equation}\label{partialxiLxi2}
    \begin{aligned}
&\left|\EE\left[X^m_i\int^{T^m+h_{i}}_{T^m}\left(\partial_{i} f(X(s))-\partial_if(X^m)\right)\rd s\middle|r^m=i\right]\right|\\
\leq & h_i\EE\left[X^m_i\left(\sup_{T^m\leq t\leq T^m+h_i}\left|\partial_{i} f(X(t))-\partial_{i} f(X^m)\right|\right)\middle|r^m=i\right]\\
\stackrel{(\mathrm{I})}{\leq} & 2h^2_i|X^m_i|^2+2\EE\left(\sup_{T^m\leq t\leq T^m+h_i}\left|\partial_{i} f(X(t))-\partial_{i} f(X^m)\right|^2\middle|r^m=i\right)\\
\stackrel{(\mathrm{II})}{\leq} & 2h^2_i|X^m_i|^2+2L^2_i\EE\left(\sup_{T^m\leq t\leq T^m+h_i}\left|X_i(t)-X^m_i\right|^2\middle|r^m=i\right)\\
\stackrel{(\mathrm{III})}{\leq} & 2h^2_i|X^m_i|^2+16h^2_iL^2_i|\partial_i f(X^m)|^2+60h_iL^2_i\\
\stackrel{(\mathrm{IV})}{\leq} & (2+16L^4_i)h^2_i|X^m_i|^2+60h_iL^2_i\,,
\end{aligned}
\end{equation}
where we used Young's inequality in (I), the Lipschitz condition in (II), Lemma \ref{lemmaA.3} below (specifically, inequality \eqref{eqn:supdifferentxi}) in (III), and the Lipschitz condition again in (IV).  This, when substituted into~\eqref{partialxiLxi}, gives
\[
2\EE\left[X^m_i\left(X^{m+1}_{i}-X^m_{i}\right)\middle|r^m=i\right]\leq -2h_iX^m_i\partial_if(X^m)+(4+32L^4_i)h^2_i|X^m_i|^2+120h_iL^2_i\,.
\]

To bound the third term in~\eqref{eqn:taylor}, again for the case $r^m=i$, we use \eqref{eqn:equationsxsw} again for:
\begin{equation}\label{eqn:discrepency_square}
\begin{aligned}
&\EE\left[\left(X^{m+1}_{i}-X^m_{i}\right)^2\middle|r^m=i\right]\\
= &\EE\left[\left(\int^{T^m+h_{i}}_{T^m}\partial_{i} f(X(s))\rd s-\sqrt{2}\int^{T^m+h_{i}}_{T^m}\rd B_s\right)^2\middle|r^m=i\right]\\
\stackrel{(\mathrm{I})}{\leq} &2h^2_i\EE\left(\sup_{T^m\leq t\leq T^m+h_i}\left|\partial_i f(X(t))\right|^2\middle|r^m=i\right)+4\EE\left(\left|\int^{T^m+h_{i}}_{T^m}\rd B_s\right|^2\middle|r^m=i\right)\\
= &2h^2_i\EE\left(\sup_{T^m\leq t\leq T^m+h_i}\left|\partial_i f(X(t))\right|^2\middle|r^m=i\right)+4h_i\\
\stackrel{(\mathrm{II})}{\leq} &8h^2_i|\partial_i f(X^m)|^2+88h_i^3L_i^2 + 4h_i\\
\stackrel{(\mathrm{III})}{\leq} &8L^2_ih^2_i|X^m_i|^2+24h_i\,,
\end{aligned}
\end{equation}
where we used Young's inequality in (I), Lemma~\ref{lemmaA.3} below (specifically, inequality \eqref{eqn:supdifferentfi}) in (II), and Lipschitz continuity in (III), together with  $88h_i^2L_i^2\leq20$ by \eqref{thm4.1hdependence}. 

Finally, we have
\[
\EE\left(L\left(X^{m+1}\right)\middle|r^m=i\right)\leq L\left(X^m\right)-2h_iX^m_i\partial_if(X^m)+(4+8L^2_i+32L^4_i)h^2_i|X^m_i|^2+(24+120L^2_i)h_i\,.
\]

By summing according to~\eqref{eqn:sumup}, and using \eqref{condition:pranh} and $L_i \le L$ for all $i=1,2,\dotsc,d$, we obtain
\[
\begin{aligned}
\EE\left(L\left(X^{m+1}\right)\right)=&\sum^d_{i=1}\phi_i\EE\left(L\left(X^{m+1}\right)\middle|r^m=i\right)\\
\leq & L\left(X^m\right)-2h\left\langle X^m,\nabla f(X^m)\right\rangle+\frac{\left(4+8L^2+32L^4\right)h^2}{\min\{\phi_i\}}(L\left(X^m\right)-1)+(24+120L^2)h\,.
\end{aligned}
\]
Finally, using
$\left\langle X^m,\nabla f(X^m)\right\rangle\geq \mu (L\left(X^m\right)-1)$ (from \eqref{Convexity.2} with $x'=X^m$ and $x=x^*=0$)
and \eqref{thm4.1hdependence}, we obtain \eqref{eqn:lyconditionlemma}.
\end{proof}

\begin{proof}[Proof of Lemma \ref{lemma:converge2}]
To prove \eqref{minorizationcondition}, we construct a new Markov process $\wsx^m$. 
Defining $\wsx^0=x^0$, we obtain $\wsx^{m+1}$ from $\wsx^m$ by running the following process:
\[
\widetilde{T}^n=\sum^{n}_{i=1}h_i,\quad \widetilde{T}^0=0,\quad Z(0)=\wsx^m\,.
\]
Then for $\widetilde{T}^{n-1}\leq t\leq \widetilde{T}^{n}$ and $n\leq d$, let 
\[
\begin{cases}
Z_n(t)=Z_n\left(\widetilde{T}^{n-1}\right)-\int^{t}_{\widetilde{T}^{n-1}}\partial_{{n}}f\left(Z(s)\right) \rd s+\sqrt{2}\int^{t}_{\widetilde{T}^{n-1}}\rd B_s\,,\\
Z_{i}(t)=Z_i\left(\widetilde{T}^{n-1}\right)\,,\quad i\neq n\,,
\end{cases}
\]
and set $\wsx^{m+1}=Z\left(\widetilde{T}^d\right)$. 
Denote the transition kernel by $\Xi_{\text{cyc}}$ (corresponding to one round of a cyclic version of the coordinate algorithm). 
We then have the following properties: 
\begin{itemize}
\item For any $x\in C$ and $A\in\mathcal{B}(\mathbb{R}^d)$, we have
\[
\Xi^d(x,A)\geq \Pi^{d}_{i=1}\phi_i \Xi_{\text{cyc}}(x,A)>0\,.
\]
\item $\Xi_{\text{cyc}}$ possesses a positive jointly continuous density.
\end{itemize}
According to \citep[Lemma~2.3]{MATTINGLY2002185},  since $\Xi_{\text{cyc}}$ has a positive jointly continuous density, there exists an $\eta'>0$ and a probability measure $\mathcal{M}$ with $\mathcal{M}(C)=1$, such that
\[
\Xi_{\text{cyc}}(x,A)>\eta' \mathcal{M}(A),\quad \forall A\in\mathcal{B}\left(\mathbb{R}^{d}\right),x\in C\,,
\]
which implies
\[
\Xi^d(x,A)\geq \Pi^{d}_{i=1}\phi_i  \Xi_{\text{cyc}}(x,A)>\Pi^{d}_{i=1}\phi_i  \eta' \mathcal{M}(A),\quad \forall A\in\mathcal{B}\left(\mathbb{R}^{d}\right),x\in C\,.
\]
This proves \eqref{minorizationcondition} by setting $\eta=\Pi^{d}_{i=1}\phi_i  \eta'$.
\end{proof}

In the proof of Lemma~\ref{lem:convergenceulmclemma}, we used several estimates in inequalities~\eqref{partialxiLxi2} and~\eqref{eqn:discrepency_square}. 
We prove these estimates in the following lemma.

\begin{lemma}\label{lemmaA.3}
Suppose that the assumptions of Lemma~\ref{lem:convergenceulmclemma} hold, and let $X_i$ evolve according to~\eqref{eqn:equationsxsw}. Then we have the following bounds:
\begin{align}
\EE\left(\sup_{T^m\leq t\leq T^m+h_i}\left|\partial_i f(X(t))\right|^2\right) & \leq 4|\partial_i f(X^m)|^2+44h_iL^2_i\,,\label{eqn:supdifferentfi}\\
\EE\left(\sup_{T^m\leq t\leq T^m+h_i}\left|X_i(t)-X^m_i\right|^2\right) & \leq 8h^2_i|\partial_i f(X^m)|^2+30h_i\label{eqn:supdifferentxi}\,.
\end{align}
\end{lemma}
\begin{proof}
To obtain~\eqref{eqn:supdifferentfi}, we have
\begin{equation}\label{supxtE}
\begin{aligned}
&\EE\left(\sup_{T^m\leq t\leq T^m+h_i}\left|\partial_i f(X(t))\right|^2\right)\\
\leq &\EE\left[\sup_{T^m\leq t\leq T^m+h_i}\left(\left|\partial_i f(X^m)\right|+L_i\left|X_i(t)-X^m_i\right|\right)^2\right]\\
\leq &2|\partial_i f(X^m)|^2+2L^2_i\EE\left(\sup_{T^m\leq t\leq T^m+h_i}\left|X_i(t)-X^m_i\right|^2\right)\,.
\end{aligned}
\end{equation}

% If~\eqref{eqn:supdifferentxi} holds true, then this term is bounded by, using $h_iL_i\leq \frac{1}{4}$,\ql{once again your constants don't match}
% \[
% \begin{aligned}
% \EE\left(\sup_{T^m\leq t\leq T^m+h_i}\left|\partial_i f(X(t))\right|^2\right)\leq &2|\partial_i f(X^m)|^2 + 2L_i^2\left(8h^2_i|\partial_i f(X^m)|^2+30h_i\right)\\
% \leq & 4|\partial_i f(X^m)|^2 + 60h_iL_i^2\,,
% \end{aligned}
% \]
% concluding~\eqref{eqn:supdifferentfi}.

To bound the second term, we use \eqref{eqn:equationsxsw} again:
\begin{equation}\label{supxtE2}
    \begin{aligned}
&\EE\left(\sup_{T^m\leq t\leq T^m+h_i}\left|X_i(t)-X^m_i\right|^2\right)\\
=&\EE\left(\sup_{T^m\leq t\leq T^m+h_i}\left|\int^{t}_{T^m}\partial_{i}f\left(X(s)\right) \rd s-\sqrt{2}\int^{t}_{T^m}\rd B_s\right|^2\right)\\
\leq &2h^2_i\EE\left(\sup_{T^m\leq t\leq T^m+h_i}\left|\partial_i f(X(t))\right|^2\right)+4\EE\left(\sup_{T^m\leq t\leq T^m+h_i}\left|\int^{t}_{T^m}\rd B_s\right|^2\right)\\
\leq &2h^2_i\EE\left(\sup_{T^m\leq t\leq T^m+h_i}\left|\partial_i f(X(t))\right|^2\right)+16h_i\,,
\end{aligned}
\end{equation}
where we use Young's inequality and 
\[
\EE\left(\sup_{T^m\leq t\leq T^m+h_i}\left|\int^{t}_{T^m}\rd B_s\right|^2\right)\leq 4\EE\left(\left|\int^{T^m+h_i}_{T^m}\rd B_s\right|^2\right)=4h_i
\]
by Doob’s maximal inequality. By substituting \eqref{supxtE2} into \eqref{supxtE}, we obtain
\[
\begin{aligned}
&\EE\left(\sup_{T^m\leq t\leq T^m+h_i}\left|\partial_i f(X(t))\right|^2\right)\\
\leq & 4h^2_iL^2_i\EE\left(\sup_{T^m\leq t\leq T^m+h_i}\left|\partial_i f(X(t))\right|^2\right)+2|\partial_i f(X^m)|^2+32h_iL^2_i\,.
\end{aligned}
\]
Using $h_iL_i\leq\frac{1}{4}$, we move the first term on the right to the left to obtain
\[
\frac{3}{4}\EE\left(\sup_{T^m\leq t\leq T^m+h_i}\left|\partial_i f(X(t))\right|^2\right)\leq 2|\partial_i f(X^m)|^2+32h_iL^2_i\,,
\]
leading to~\eqref{eqn:supdifferentfi}. 
Then we obtain~\eqref{eqn:supdifferentxi} by plugging this in~\eqref{supxtE2} and using the fact that $88h_i^3L_i^2<14h_i$ by \eqref{thm4.1hdependence}.
\end{proof}

\section{Proof of Theorem \ref{thm:rcolmc}}\label{sec:appendixA}
The proof of this theorem requires us to design a reference solution to explicitly bound $W(q_m,p)$. 
Let $\wx^0$ be a random vector drawn from target distribution induced by $p$, so that $W^2_2(q_0,p)=\EE|x^0-\wx^0|^2$. 
We then require $\wx$ to solve the  following SDE: for $t \in (T^m, T^{m+1}]$, with $T^m$ defined in~\eqref{palphasalpha}:
\begin{equation}\label{eqn:yt}
\left\{
\begin{aligned}
\wx_{r^m}(t) &=\wx_{r^m}({T^m})-\int^{t}_{T^m} \partial_{r^m} f(\wx(s))\rd s+\sqrt{2}\int^t_{T^m}\rd B_s\,,\\
\wx_{i}(t) &=\wx_{i}({T^m}),\quad i\neq r^m\,.
\end{aligned}
\right.
\end{equation}
If we use the same Brownian motion as in \eqref{alg:updatexm}, we have
\begin{equation}\label{eqn:ymolmc}
\wx^{m+1}=\wx^{m}+\left[-\int^{T^{m+1}}_{T^m} \partial_{r^m} f(\wx(s))\rd s+\sqrt{2h_{r^m}}\xi^{m}\right]\boldsymbol{e}_{r^m}\,,
\end{equation}
where $\boldsymbol{e}_{r^m}$ is the unit vector in $r^m$ direction. 
Since the $r^m$-th marginal distribution of $\wx(t)$ is preserved in each time step according to~\eqref{eqn:yt}, the whole distribution of $\wx(t)$ is preserved to be $p$ for all $t$. 
Therefore, by the definition $W_m=W(q_m,p)$, we have
\[
W^2_m\leq\EE|\Delta^m|^2=\EE|x^m-\wx^m|^2\,,
\]
where
\begin{equation}\label{eqn:deltaolmc}
\Delta^m:=\wx^m-x^m\,.
\end{equation}
This means bounding $W_m$ amounts to evaluating $\EE|\Delta^m|^2$. Under Assumption~\ref{assum:Cov}, we have the following result.

\begin{proposition}\label{prop:rcolmc}
Suppose the assumptions of Theorem~\ref{thm:rcolmc} are satisfied and let $\{x^m\}$, $\{ \wx^m \}$, and $\{\Delta^m\}$  be defined in \eqref{alg:updatexm}, \eqref{eqn:yt}, and \eqref{eqn:deltaolmc}, respectively.
Then, we have
\begin{equation}\label{eqn:itrcolmc}
\EE|\Delta^{m+1}|^2\leq \left(1-\frac{h\mu}{2}\right)\EE|\Delta^m|^2+\frac{10h^2}{\mu}\sum^d_{i=1}\frac{L^2_i}{\phi_i}\,.
\end{equation}
\end{proposition}
The proof of this result appears in Appendix~\ref{sec:proofofthm:rcolmc}. 
The proof for Theorem~\ref{thm:rcolmc} is now immediate.
\begin{proof}[Proof of Theorem \ref{thm:rcolmc}]
By iterating \eqref{eqn:itrcolmc}, we obtain
\[
\EE|\Delta^{m}|^2\leq \left(1-\frac{h\mu}{2}\right)^m\EE|\Delta^0|^2+\frac{20h}{\mu^2}\sum^d_{i=1}\frac{L^2_i}{\phi_i}\,,
\]
and since $h \mu/2  \in (0,1)$, we have
\begin{equation}\label{eqn:itrcolmc2}
\begin{aligned}
\EE|\Delta^{m}|^2\leq \exp\left(-\frac{\mu hm}{2}\right)\EE|\Delta^0|^2+\frac{20h}{\mu^2}\sum^d_{i=1}\frac{L^2_i}{\phi_i}\,.
\end{aligned}
\end{equation}
By construction, we have $W^{2}(q_0,p) = \EE|\Delta^0|^2$ and $W^{2}(q_m,p)\leq \EE|\Delta^{m}|^2$.
By taking the square root of both sides and using $a^2 \le b^2 + c^2 \Rightarrow a \le b+c$ for any nonnegative $a$, $b$, and $c$, we arrive at \eqref{eqn:convergencercolmc}.
\end{proof}

The proof for Corollary~\ref{cor:first} is also obvious.
\begin{proof}[Proof of Corollary~\ref{cor:first}]
To ensure that $W_m\leq\epsilon$, we set the two terms on the right hand side of~\eqref{eqn:convergencercolmc} to be smaller than $\epsilon/2$, which implies that
\begin{equation}\label{eqn:olmccomputationalcomplexity}
h=O\left(\frac{\mu^2\epsilon^2}{100\sum^d_{i=1}\frac{L^2_i}{\phi_i(\alpha)}}\right) \quad \text{and} \quad m\geq \frac{4}{\mu h}\log\left(\frac{2W_0}{\epsilon}\right)\,.
\end{equation}
By using the definition of  $\phi_i(\alpha)$ according to \eqref{condition:prand}, we obtain
\[
\sum^d_{i=1}\frac{L^2_i}{\phi_i(\alpha)}= \left(\sum_{i=1}^d \frac{L_i^{2}}{L_i^{\alpha}}\right)\left(\sum_{j=1}^dL_j^\alpha\right) = \mu^2K_{2-\alpha}K_\alpha\,,
\]
which implies that $m= \widetilde{O}\left(\left(K_{2-\alpha}K_\alpha\right)/(\mu\epsilon^2)\right)$. 
Furthermore, $\alpha =1$ gives the optimal cost, because:
\[
K_{2-\alpha}K_{\alpha} = \left(\sum\kappa_i^\alpha\right)\left(\sum\kappa_i^{2-\alpha}\right)\geq \left(\sum_i\kappa_i\right)^2 = K_1^2\,,
\]
due to H\"older's inequality.
\end{proof}

\subsection{Proof of Proposition~\ref{prop:rcolmc}}\label{sec:proofofthm:rcolmc}
We prove the Proposition by means of the following lemma.
\begin{lemma}\label{lem:rcolmc}
Under the conditions of Proposition \ref{prop:rcolmc}, for $m\geq 0$ and $i=1,2,\dots,d$, we have
\begin{equation}\label{eqn:resultoflem}
\begin{aligned}
\EE|\Delta^{m+1}_i|^2 & \leq \left(1+h\mu+\frac{h^2\mu^2}{\phi_i}\right)\EE|\Delta^m_i|^2 -2h\EE\left[\Delta^m_i\left(\partial_i f(\wx^m)-\partial_i f(x^m)\right)\right]\\
&\qquad +\frac{3h^2}{\phi_i}\EE\left|\partial_i f(\wx^m)-\partial_i f(x^m)\right|^2 +\left(\frac{2h^3L^3_i}{\mu \phi^2_i} +\frac{8h^2L^{2}_i}{\mu \phi_i}\right)\,.
\end{aligned}
\end{equation}
\end{lemma}
\begin{proof}
In the $m$-th time step, we have
\[
\mathbb{P}(r^m=i)=\phi_i,\quad \mathbb{P}(r^m\neq i)=1-\phi_i\,,
\]
so that
\begin{equation}\label{eqn:pickr}
\begin{aligned}
\EE|\Delta^{m+1}_i|^2&=\phi_i\EE\left(|\Delta^{m+1}_i|^2\mid r^m=i\right)+\left(1-\phi_i\right)\EE\left(|\Delta^{m+1}_i|^2 \mid r^m\neq i\right)\\&=\phi_i\EE\left(|\Delta^{m+1}_i|^2\mid r^m=i\right)+\left(1-\phi_i\right)\EE\left|\Delta^{m}_i\right|^2\,.
\end{aligned}
\end{equation}
We now  analyze the first term on the right hand side under condition $r^m=i$. By definition of $\Delta^{m+1}_i$, we have
\begin{equation}\label{eqn:Deltam+1}
\begin{aligned}
\Delta^{m+1}_i&=\Delta^m_i+(\wx^{m+1}_i-\wx^{m}_i) -(x^{m+1}_i-x^m_i)\\
&=\Delta^m_i+\left(-\int^{T^{m}+h_i}_{T^m}\partial_i f(\wx(s))\rd s+\sqrt{2h_i}\xi_m\right)-\left(-\int^{T^{m}+h_i}_{T^m}\partial_i f(x^m)\rd s+\sqrt{2h_i}\xi_m\right)\\
&=\Delta^m_i-\int^{T^{m}+h_i}_{T^m}\left(\partial_i f(\wx(s))-\partial_i f(x^m)\right)\rd s\\
&=\Delta^m_i-\int^{T^{m}+h_i}_{T^m}\left(\partial_i f(\wx(s))-\partial_i f(\wx^m)+\partial_i f(\wx^m)-\partial_i f(x^m)\right)\rd s \\
&=\Delta^m_i-h_i\left(\partial_i f(\wx^m)-\partial_i f(x^m)\right)-\int^{T^{m}+h_i}_{T^m}\left(\partial_i f(\wx(s))-\partial_i f(\wx^m)\right)\rd s\\
& =\Delta^m_i-h_i\left(\partial_i f(\wx^m)-\partial_i f(x^m)\right) - V^m\,,
\end{aligned}
\end{equation}
where we have defined
\begin{equation}\label{eqn:Vm}
V^m:=\int^{T^{m}+h_i}_{T^{m}}\left(\partial_i f(\wx(s))-\partial_i f(\wx^m)\right)\rd s\,.
\end{equation}
% This implies
% \[
% \left|\Delta^{m+1}_i\right|^2=|\Delta^{m+1}_i+V^m-V^m|^2=|\Delta^{m+1}_i+V^m|^2-2\left\langle \Delta^{m+1}_i+V^m,V^m\right\rangle+|V^m|^2.
% \]
% \sw{the "$+2$" in the formula above should be "$-2$", but I don't see where this formula is used anyway. Should it be removed?}~\zd{Yes, fixed. I write this down because I want to explain proof process below.} 
% We will show below that $V^m$ is small (on the order of $h^3$), and the first term can be controlled by $\Delta^m$ using the convex and gradient Lipschitz condition. To control the cross term, we employ
By Young's inequality, we have
\begin{align}
\nonumber
& \EE\left(|\Delta^{m+1}_i|^2\mid r^m=i\right) \\
\nonumber 
& =\EE\left(|\Delta^{m+1}_i+V^m-V^m|^2\mid r^m=i\right)\\
\label{Deltam+1}
& \leq (1+a)\, \EE\left(|\Delta^{m+1}_i+V^m|^2\mid r^m=i\right)+\left(1+\frac{1}{a}\right)\EE\left(|V^m|^2 \mid r^m=i\right)\,, 
\end{align}
where $a > 0$ is a parameter to be specified later. 

For the first term on the right hand side of \eqref{Deltam+1}, we have
\begin{align}
\nonumber
& \EE\left(|\Delta^{m+1}_i+V^m|^2\mid r^m=i\right)\\
\nonumber
&=\EE|\Delta^m_i-h_i\left(\partial_i f(\wx^m)-\partial_i f(x^m)\right)|^2\\
\label{firsttermDeltam+1}
&=\EE|\Delta^m_i|^2-2h_i\EE\left[\Delta^m_i\left(\partial_i f(\wx^m)-\partial_i f(x^m)\right)\right]+h^2_i\EE\left|\partial_i f(\wx^m)-\partial_i f(x^m)\right|^2\,.
\end{align}
Note that the second term will essentially become the second line in~\eqref{eqn:resultoflem}, and the third term will become the third line in~\eqref{eqn:resultoflem} (upon the proper choice of $a$). For very small $h$, this term is negligible.

For the second term on the right-hand side of \eqref{Deltam+1}, we recall the definition \eqref{eqn:Vm} and obtain
\begin{align}
\nonumber
\EE\left(|V^m|^2\middle| r^m=i\right)&\stackrel{\text{(I)}}{\leq} h_i \int^{T^{m}+h_i}_{T^m}\EE\left(\left|\partial_i f(\wx(s))-\partial_i f(\wx^m)\right|^2\middle| r^m=i\right)\rd s\\
\nonumber
&\stackrel{\text{(II)}}{\leq} h_iL^2_i \int^{T^{m}+h_i}_{T^m}\EE\left(\left|\wx(s)-\wx^m\right|^2\middle| r^m=i\right)\rd s\\
\nonumber
&= h_iL^2_i \int^{T^{m}+h_i}_{T^m}\EE\left(\left|\int^{s}_{T^m} \partial_i f(\wx(t))\rd t+\sqrt{2}(B_s-B_{T^m})\right|^2\middle| r^m=i\right)\rd s\\
\nonumber
&\stackrel{\text{(III)}}{\leq} 2h^2_iL^2_i \int^{T^m+h_i}_{T^m}\int^{s}_{T^m} \EE\left(\left|\partial_i f(\wx(t))\right|^2\middle| r^m=i\right)\rd t\rd s \\
\nonumber
& \quad\quad +4h^2_iL^2_i\int^{T^m+h_i}_{T^m} \EE|\xi^m|^2\rd s\\
\nonumber
&\stackrel{\text{(IV)}}{=} h^4_iL^2_i\EE\left(\left|\partial_i f(\wx^m)\right|^2\right)+4h^3_iL^2_i\\
\label{secondtermDeltam+1}
&\stackrel{\text{(V)}}{=}h^4_iL^2_i\EE_p|\partial_if|^2+4h^3_iL^2_i\stackrel{\text{(VI)}}{\leq} h^4_iL^3_i+4h^3_iL^2_i\,,
\end{align}
where $\text{(II)}$ comes from $L$-Lipschitz condition~\eqref{GradientLipcoord}, $\text{(I)}$ and $\text{(III)}$ come from the use of Young's inequality and                                                                                         Jensen's inequality when we move the $|\cdot|^2$ from outside to inside of the integral, and $\text{(IV)}$ and $\text{(V)}$ hold true because $\wx(t)\sim p$ for all $t$. In $\text{(VI)}$ we use $\EE_p|\partial_if|^2\leq L_i$ using~\citep[Lemma~3]{DALALYAN20195278}.

By substituting \eqref{firsttermDeltam+1} and \eqref{secondtermDeltam+1} into the right hand side of \eqref{Deltam+1}, we obtain
\begin{align}
\nonumber
&\EE\left(|\Delta^{m+1}_i|^2\mid r^m=i\right)\\
\nonumber
&\leq (1+a)\EE|\Delta^m_i|^2 -2h_i(1+a)\EE\left[\Delta^m_i\left(\partial_i f(\wx^m)-\partial_i f(x^m)\right)\right]\\
\label{eqn:Deltam+1new1}
&\quad\quad +h^2_i(1+a)\EE\left|\partial_i f(\wx^m)-\partial_i f(x^m)\right|^2 +\left(1+\frac{1}{a}\right)\left(h^4_iL^3_i+4h^3_iL^2_i\right)\,.
\end{align}
By substituting \eqref{eqn:Deltam+1new1} into \eqref{eqn:pickr}, we have 
\begin{align}
\nonumber
\EE|\Delta^{m+1}_i|^2& \leq \left(1+a\phi_i\right)\EE|\Delta^m_i|^2 -2(1+a)h\EE\left[\Delta^m_i\left(\partial_i f(\wx^m)-\partial_i f(x^m)\right)\right]\\
\label{eqn:Deltam+1new1nonuni}
&\qquad +\frac{(1+a)h^2}{\phi_i}\EE\left|\partial_i f(\wx^m)-\partial_i f(x^m)\right|^2 +\left(1+\frac{1}{a}\right)\left(\frac{h^4L^3_i}{\phi^3_i}+\frac{4h^3L^2_i}{\phi^2_i}\right)\,,
\end{align}
where we have used $h_i\phi_i=h$.

Now, we need to choose a value of $a>0$ appropriate to establish \eqref{eqn:resultoflem}. 
By comparing the two formulas, we see the need to set
\[
a\phi_i = h\mu\,\quad\Rightarrow\quad a=h_i\mu=\frac{h\mu}{\phi_i}\leq 1\,.
\]
since $h\leq {\min\{\phi_i\}}/{\mu}$. 
It follows that $1+\frac{1}{a}\leq \frac{2\phi_i}{h\mu}$. 
By substituting into \eqref{eqn:Deltam+1new1nonuni}, we obtain
%If we ignore the different coefficients and sum second term of \eqref{eqn:Deltam+1new1nonuni} up with $i$, we can obtain $\left\langle \Delta^m,\nabla f(\wx^m)-\nabla f(x^m)\right\rangle$, which is negative by convexity of $f$. This implies the second term will give us a decay term with coefficient $\mathcal{O}(h)$ before it. Therefore, we choose
%\[
%a=h_i\mu=\frac{h\mu}{\phi_i}<1\,,
%\]
%where we use $h<\frac{\min\{\phi_i\}}{\mu}$. This also implies
\begin{align}
\nonumber
\EE|\Delta^{m+1}_i|^2& \leq \left(1+h\mu\right)\EE|\Delta^m_i|^2 -2h\EE\left[\Delta^m_i\left(\partial_i f(\wx^m)-\partial_i f(x^m)\right)\right]\\
\nonumber
&\qquad -\frac{2h^2\mu}{\phi_i} \EE\left[\Delta^m_i\left(\partial_i f(\wx^m)-\partial_i f(x^m)\right)\right] +\frac{2h^2}{\phi_i}\EE\left|\partial_i f(\wx^m)-\partial_i f(x^m)\right|^2 \\
\label{eqn:Deltam+1new1nonuni2}
&\qquad +\left(\frac{2h^3L^3_i}{\mu \phi^2_i} +\frac{8h^2L^{2}_i}{\mu \phi_i}\right)\,.
\end{align}
We conclude the lemma by using the following Cauchy-Schwartz inequality to control the third term on the right hand side of this expression:
\begin{equation*}
-\frac{2h^2\mu}{\phi_i} \EE\left[\Delta^m_i\left(\partial_i f(\wx^m)-\partial_i f(x^m)\right)\right]\leq \frac{h^2\mu^2}{\phi_i}\EE|\Delta^m_i|^2+\frac{h^2}{\phi_i}\EE|\partial_i f(\wx^m)-\partial_i f(x^m)|^2\,. \qedhere
\end{equation*}
\end{proof}

Proposition \ref{prop:rcolmc} is obtained by simply summing all components in the lemma.

\begin{proof}[Proof of Proposion \ref{prop:rcolmc}]
Noting
\[
\EE|\Delta^{m+1}|^2=\sum^d_{i=1}\EE|\Delta^{m+1}_i|^2\,,
\]
we bound the right hand side by \eqref{eqn:resultoflem} and get
\begin{equation}\label{Prop:Deltam+1}
\begin{aligned}
\EE|\Delta^{m+1}|^2& \leq \left(1+h\mu+\frac{h^2\mu^2}{\min\{\phi_i\}}\right)\EE|\Delta^m|^2 -2h\EE\left\langle\Delta^m,\nabla f(\wx^m)-\nabla f(x^m)\right\rangle\\
&\qquad +\frac{3h^2}{\min\{\phi_i\}}\EE\left|\nabla f(\wx^m)-\nabla f(x^m)\right|^2 +\left(\frac{2h^3}{\mu}\sum^{d}_{i=1} \frac{L^3_i}{\phi^2_i}+\frac{8h^2}{\mu}\sum^{d}_{i=1}\frac{L^2_i}{\phi_i}\right)\,.
\end{aligned}
\end{equation}
The second and third terms on the right-hand side can be bounded in terms of  $\EE|\Delta^m|^2$: 
\begin{itemize}
\item By convexity, we have
\begin{equation}\label{Prop:Deltam+1second}
\EE\left\langle \Delta^m,\nabla f(\wx^m)-\nabla f(x^m)\right\rangle\geq \mu\EE|\Delta^m|^2\,.
\end{equation}
\item As the gradient is $L$-Lipschitz, we have 
\begin{equation}\label{Prop:Deltam+1third}
\EE\left|\nabla f(\wx^m)-\nabla f(x^m)\right|^2\leq L^2\EE|\Delta^m|^2\,.
\end{equation}
\end{itemize}
By substituting \eqref{Prop:Deltam+1second} and \eqref{Prop:Deltam+1third} into \eqref{Prop:Deltam+1} and using $\mu\leq L$, we obtain
\begin{equation}\label{eqn:itrcolmcnonuni}
	\EE|\Delta^{m+1}|^2\leq \left(1-h\mu+\frac{4h^2L^2}{\min\{\phi_i\}}\right)\EE|\Delta^m|^2 +\left(\frac{2h^3}{\mu}\sum^{d}_{i=1} \frac{L^3_i}{\phi^2_i}+\frac{8h^2}{\mu}\sum^{d}_{i=1}\frac{L^2_i}{\phi_i}\right)\,.
\end{equation}
If we take $h$ sufficiently small, the coefficient in front of $\EE|\Delta^m|^2$ is strictly smaller than $1$, ensuring the decay of the error. Indeed, by setting $h\leq \frac{\mu\min\left\{\phi_i\right\}}{8L^2}$, we have
\[
\frac{4h^2L^2}{\min\{\phi_i\}}\leq \frac{h\mu}{2},\quad\text{and}\quad \frac{h L_i}{\phi_i}\leq \frac{\mu}{8L} \leq 1\,,
\]
which leads to the iteration formula~\eqref{eqn:itrcolmc}.
\end{proof}

\section{Proof of Theorem \ref{thm:rcolmc2}}\label{sec:appendixB}
Theorem~\ref{thm:rcolmc2} is based on the following proposition.
\begin{proposition}\label{prop:rcolmc2}
Suppose the assumptions of Theorem~\ref{thm:rcolmc2} and let $\{x^m\}$, $\{ \wx^m\}$, and $\{\Delta_m\}$ be defined as in \eqref{alg:updatexm}, \eqref{eqn:yt}, and \eqref{eqn:deltaolmc}, respectively. Then we have
\begin{equation}\label{eqn:itrcolmc22}
\EE|\Delta^{m+1}|^2\leq \left(1-\frac{h\mu}{2}\right)\EE|\Delta^m|^2+\frac{4h^3}{\mu}\sum^d_{i=1}\frac{\left(L^3_i+H^2_i\right)}{\phi^2_i}\,.
\end{equation}
\end{proposition}
We prove this result in Appendix~\ref{sec:proofofthm:rcolmc2}. The proof of the theorem is now immediate.
\begin{proof}[Proof of Theorem~\ref{thm:rcolmc2}]
Use \eqref{eqn:itrcolmc22} iteratively, we have
% \begin{equation}\label{eqn:itrcolmc23}
\begin{align*}
\EE|\Delta^{m+1}|^2& \leq \left(1-\frac{h\mu}{2}\right)^m\EE|\Delta^0|^2+\frac{8h^2}{\mu^2}\sum^d_{i=1}\frac{\left(L^3_i+H^2_i\right)}{\phi^2_i} \\
& \leq \exp\left(-\frac{\mu hm}{2}\right)\EE|\Delta^0|^2+\frac{8h^2}{\mu^2}\sum^d_{i=1}\frac{\left(L^3_i+H^2_i\right)}{\phi^2_i}\,.
\end{align*}
Using  $W^{2}(q_0,p) = \EE|\Delta^0|^2$ and $W^{2}(q_m,p)\leq \EE|\Delta^{m}|^2$, we take the square root on both sides, we obtain \eqref{eqn:convergencercolmc2}.
\end{proof}

The proof of Corollary~\ref{cor:second} is also immediate.
\begin{proof}[Proof of Corollary~\ref{cor:second}] Use \eqref{eqn:convergencercolmc2}, to ensure $W_m\leq \epsilon$, we set two terms on the right hand side of \eqref{eqn:convergencercolmc2} to be smaller than $\epsilon/2$, which implies that
\begin{equation}\label{eqn:hmcorro}
  h=O\left(\frac{\epsilon \mu}{\sqrt{\sum^d_{i=1}\frac{\left(L^3_i+H^2_i\right)}{\phi_i^2}}}\right),\quad m\geq \frac{4}{\mu h}\log\left(\frac{2W_0}{\epsilon}\right)\,.  
\end{equation}

To find optimal choice of $\phi_i$, we need to minimize
\[
\sum^d_{i=1}\frac{\left(L^3_i+H^2_i\right)}{\phi_i^2}
\]
under constraint $\sum^d_i\phi_i=1$ and $\phi_i>0$. 
Introducing a Lagrange multiplier $\lambda \in \R$, define the Lagrangian function as follows:
\[
F(\phi_1,\phi_2,\dots,\phi_d,\lambda)=\sum^d_{i=1}\frac{\left(L^3_i+H^2_i\right)}{\phi_i^2}+\lambda\left(\sum^d_{i=1}\phi_i-1\right)\,.
\]
By setting $\partial F/\partial \phi_i=0$ for all $i$, and substituting into the constraint $\sum^d_i\phi_i=1$  to find the appropriate value of $\lambda$, we find that the optimal $(\phi_1,\phi_2,\dotsc,\phi_d)$ satisfies
\[
\phi_i=\frac{\left(L_i^3 + H_i^2\right)^{1/3}}{\sum^d_{i=1}\left(L_i^3 + H_i^2\right)^{1/3}},\quad i=1,2,\dotsc,d.
\]
By substituting  into \eqref{eqn:hmcorro}, we obtain \eqref{eqn:m_Hessian}.
\end{proof}
\subsection{Proof of Proposition~\ref{prop:rcolmc2}}\label{sec:proofofthm:rcolmc2}
The strategy of the proof for this proposition is almost identical to that of the previous section. The reference solution $\wx$ is defined as in  \eqref{eqn:yt}. We will use the following lemma:
\begin{lemma}\label{lem:rcolmc2}
Under the conditions of Proposition \ref{prop:rcolmc2}, for $m\geq 0$ and $i=1,2,\dots,d$, we have
\begin{align}
\nonumber
\EE|\Delta^{m+1}_i|^2\leq &\left(1+h\mu+\frac{h^2\mu^2}{\phi_i}\right)\EE|\Delta^m_i|^2-2h\EE\left[\Delta^m_i\left(\partial_i f(\wx^m)-\partial_i f(x^m)\right)\right]\\
\label{eqn:resultoflem2}
&+\frac{3h^2}{\phi_i}\EE\left|\partial_i f(\wx^m)-\partial_i f(x^m)\right|^2
+\frac{4h^3\left(L^3_i+H^2_i\right)}{\phi^2_i\mu}\,.
\end{align}
\end{lemma}
\begin{proof}
In the $m$-th time step, we have
\[
\mathbb{P}(r^m=i)=\phi_i,\quad \mathbb{P}(r^m\neq i)=1-\phi_i\,,
\]
meaning that
\begin{equation}\label{eqn:pickr2}
\begin{aligned}
\EE|\Delta^{m+1}_i|^2&=\phi_i\EE\left(|\Delta^{m+1}_i|^2\mid r^m=i\right)+\left(1-\phi_i\right)\EE\left(|\Delta^{m+1}_i|^2\mid r^m\neq i\right)\\&=\phi_i\EE\left(|\Delta^{m+1}_i|^2\mid r^m=i\right)+\left(1-\phi_i\right)\EE\left|\Delta^{m}_i\right|^2\,.
\end{aligned}
\end{equation}

To bound the first term in \eqref{eqn:pickr} we use the definition of $\Delta^{m+1}_i$. Under the condition $r^m=i$, we have, with the same derivation as in~\eqref{eqn:Deltam+1}:
\begin{equation}\label{eqn:Deltam+12}
\begin{aligned}
\Delta^{m+1}_i&=\Delta^m_i-h_i\left(\partial_i f(\wx^m)-\partial_i f(x^m)\right)-\int^{T^{m}+h_i}_{T^m}\left(\partial_i f(\wx(s))-\partial_i f(\wx^m)\right)\rd s\\
&=\Delta^m_i-h_i\left(\partial_i f(\wx^m)-\partial_i f(x^m)\right)-V^m\,,
\end{aligned}
\end{equation}
where we denoted $V^m=\int^{T^{m}+h_i}_{T^m}\left(\partial_i f(\wx(s))-\partial_i f(\wx^m)\right)\rd s$.

However, different from \eqref{secondtermDeltam+1}, since $f$ has higher regularity, we can find a tighter bound for the integral. Denote
\begin{equation}\label{eqn:Um}
U^m=\int^{T^{m}+h_i}_{T^m}\left(\partial_i f(\wx(s))-\partial_i f(\wx^m)-\sqrt{2}\int^s_{T^m}\partial_{ii}f(\wx(z))\rd B_z\right)\rd s
\end{equation}
and
\begin{equation}\label{eqn:Phim}
\Phi^m=\sqrt{2}\int^{T^{m}+h_i}_{T^m}\int^s_{T^m}\partial_{ii}f(\wx(z))\rd B_z\rd s\,.
\end{equation}
Then \eqref{eqn:Deltam+12} can be written as
\begin{equation}\label{eqn:delta_m+1+U}
\Delta^{m+1}_i=\Delta^m_i-h_i\left(\partial_i f(\wx^m)-\partial_i f(x^m)\right)-\Phi^m-U^m\,,
\end{equation}
which implies, according to Young's inequality, that, for any $a$:
\begin{equation}\label{Deltam+12}
\begin{aligned}
&\EE\left(|\Delta^{m+1}_i|^2\middle| r^m=i\right)=\EE\left(|\Delta^{m+1}_i+U^m-U^m|^2\middle| r^m=i\right)\\
\leq &(1+a)\EE\left(|\Delta^{m+1}_i+U^m|^2\middle| r^m=i\right)+\left(1+\frac{1}{a}\right)\EE\left(|U^m|^2\middle| r^m=i\right)\,.
\end{aligned}
\end{equation}
Both terms on the right-hand side of  \eqref{Deltam+12} are small. We now control the first term. Plug in the definition~\eqref{eqn:delta_m+1+U}, we have:
\begin{equation}\label{eqn:DeltaU_2}
\EE\left(|\Delta^{m+1}_i+U^m|^2\mid r^m=i\right) =\EE\left(|\Delta^m_i-h_i\left(\partial_i f(\wx^m)-\partial_i f(x^m)\right)-\Phi^m|^2\middle| r^m=i\right)\,.
\end{equation}
Noting that
\[
\begin{aligned}
&\EE\left(\left( \Delta^m_i-h_i\left(\partial_i f(\wx^m)-\partial_i f(x^m)\right)\right)\cdot \Phi^m \right)\\
=&\sqrt{2}\int^{T^{m}+h_i}_{T^m}\EE\left[\int^s_{T^m}\left( \Delta^m_i-h_i\left(\partial_i f(\wx^m)-\partial_i f(x^m)\right)\right)\cdot\partial_{ii}f(\wx(z))\rd B_z\right]\rd s=0
\end{aligned}
\]
because
\[
\EE\left[\int^s_{T^m}\left( \Delta^m_i-h_i\left(\partial_i f(\wx^m)-\partial_i f(x^m)\right)\right)\cdot\partial_{ii}f(\wx(z))\rd B_z\right]=0\,,
\]
according to the property of It\^o's integral, we can discard the cross terms with $\Phi^m$ in~\eqref{eqn:DeltaU_2} to obtain
\begin{align}\label{firsttermDeltam+12}
\EE\left(|\Delta^{m+1}_i+U^m|^2\mid r^m=i\right)
& =\EE|\Delta^m_i|^2-2h_i\EE\left[\Delta^m_i\left(\partial_i f(\wx^m)-\partial_i f(x^m)\right)\right]\nonumber\\
& \quad +h^2_i\EE\left|\partial_i f(\wx^m)-\partial_i f(x^m)\right|^2+\EE\left(|\Phi^m|^2\middle| r^m=i\right)\,.
\end{align}
For the last term of \eqref{firsttermDeltam+12}, we have the following control:
\begin{align*}
\EE\left(|\Phi^m|^2\middle| r^m=i\right)= &\EE\left(2\left|\int^{T^{m}+h_i}_{T^m}\int^s_{T^m}\partial_{ii}f(\wx(z))\rd B_z\rd s\right|^2\middle| r^m=i\right)\\
\stackrel{\text{(I)}}{\leq}&2\EE\left[\left(\int^{T^{m}+h_i}_{T^m} \rd s\right)\left(\int^{T^{m}+h_i}_{T^m}\left|\int^s_{T^m}\partial_{ii}f(\wx(z))\rd B_z\right|^2\rd s\right)\middle| r^m=i\right]\\
\leq &2h_i\int^{T^m+h_i}_{T^m}\EE\left(\left|\int^s_{T^m}\partial_{ii}f(\wx(z))\rd B_z\right|^2\middle| r^m=i\right)\rd s\\
\stackrel{\text{(II)}}{=}&2h_i\int^{T^m+h_i}_{T^m}\int^s_{T^m}\EE\left(\left|\partial_{ii}f(\wx(z))\right|^2\middle| r^m=i\right)\rd z\rd s\\
\stackrel{\text{(III)}}{=} &h^3_i\EE_p|\partial_{ii}f|^2=h^3_iL^2_i\,,
\end{align*}
where we use H\"older's inequality in $\mathrm{I}$ and $\wx(t)\sim p$ for all $t$ in $\mathrm{III}$. In $\mathrm{II}$, we use the following property of It\^o's integral:
\[
\EE\left(\left|\int^s_{T^m}\partial_{ii}f(\wx(z))\rd B_z\right|^2\middle| r^m=i\right)=\int^s_{T^m}\EE\left(\left|\partial_{ii}f(\wx(z))\right|^2\middle| r^m=i\right)\rd z\,.
\]
By substituting into \eqref{firsttermDeltam+12}, we obtain
\begin{align}
\nonumber
\EE\left(|\Delta^{m+1}_i+U^m|^2\mid r^m=i\right)\leq &\EE|\Delta^m_i|^2-2h_i\EE\left[\Delta^m_i\left(\partial_i f(\wx^m)-\partial_i f(x^m)\right)\right]\\
\label{firsttermDeltam+13}
&+h^2_i\EE\left|\partial_i f(\wx^m)-\partial_i f(x^m)\right|^2+h^3_iL^2_i
\end{align}

To bound the second term on the right-hand side of \eqref{Deltam+12}, we first note that $f$ is three times continuously differentiable, and~\eqref{HessisnLipcoord} implies $\|\partial_{iii}f\|_\infty\leq H_i$. Take $\rd t$ on both sides of ~\eqref{eqn:yt}, under condition $r^m=i$, we first have
\begin{equation}\label{Ito1}
\rd\wx_{i}(t)=-\partial_{i} f(\wx(s))\rd s+\sqrt{2}\rd B_s\,.
\end{equation}
According to It\^o's formula, we obtain
\begin{equation}\label{Ito}
\partial_i f(\wx(t))-\partial_i f(\wx^m)=\int^t_{T^m}\partial_{ii}f(\wx(s))\rd\wx_i(s)+\int^t_{T^m} \partial_{iii}f(\wx(s))\rd s\,.
\end{equation}
Substituting  \eqref{Ito1} into \eqref{Ito}, we have
\begin{equation}\label{Ito2}
\begin{aligned}
& \partial_i f(\wx(t))-\partial_i f(\wx^m)-\sqrt{2}\int_{T_m}^t\partial_{ii}f(\tilde{x}(s))\rd B_s\\
& =\int^t_{T^m}-\partial_{ii}f(\wx(s))\partial_if(\wx(s))+\partial_{iii}f(\wx(s))\rd s\,.
\end{aligned}
\end{equation}
By substituting into \eqref{eqn:Um}, we obtain
\begin{align}
\nonumber
& \EE\left(|U^m|^2\mid r^m=i\right) \\
\nonumber
\stackrel{\text{(I)}}{\leq} &h_i \int^{T^{m}+h_i}_{T^m}\EE\left(\left|\partial_i f(\wx(s))-\partial_i f(\wx^m)-\sqrt{2}\int^s_{T^m}\partial_{ii}f(\wx(z))\rd B_r\right|^2\middle| r^m=i\right)\rd s\\
\nonumber
\stackrel{\text{(II)}}{=q} &h_i \int^{T^{m}+h_i}_{T^m}\EE\left(\left|\int^s_{T^m}\left(-\partial_{ii}f(\wx(z))\partial_if(\wx(z))+\partial_{iii}f(\wx(z))\right)\rd z\right|^2\middle| r^m=i\right)\rd s\\
\nonumber
\stackrel{\text{(III)}}{\leq} &h_i^2 \int^{T^{m}+h_i}_{T^m}\int^s_{T^m}\EE\left(\left|\partial_{ii}f(\wx(z))\partial_if(\wx(z))+\partial_{iii}f(\wx(z))\right|^2\middle| r^m=i\right)\rd z\rd s\\
\nonumber
\stackrel{\text{(IV)}}{\leq} &2h_i^2 \int^{T^{m}+h_i}_{T^m}\int^s_{T^m}\EE\left(\left|\partial_{ii}f(\wx(z))\partial_if(\wx(z))\right|^2\middle| r^m=i\right)\rd z\rd s\\
\nonumber
&+2h_i^2 \int^{T^{m}+h_i}_{T^m}\int^s_{T^m}\EE\left(\left|\partial_{iii}f(\wx(z))\right|^2\middle| r^m=i\right)\rd z\rd s\\
\label{secondtermDeltam+1Um}
\stackrel{\text{(V)}}{\leq} &h_i^4\left(L^3_i+H^2_i\right)\,.
\end{align}
In the derivation, $\text{(II)}$ comes from plugging in~\eqref{Ito2}, and $\text{(I)}$ and $\text{(III)}$ come from the use of Jensen's inequality, $\text{(V)}$ comes from the use of Lipschitz continuity in the first and the second derivative (\eqref{GradientLipcoord} and \eqref{HessisnLipcoord} in particular), and the fact that $\wx(t)\sim p$ for all $t$. Note also $\EE_p|\partial_if|^2\leq L_i$ by \citep[Lemma~3]{DALALYAN20195278}.

%If $f$ is not three times continuously differentiable, we still have same bound as \eqref{secondtermDeltam+1Um}. The proof is similar to \citep{DALALYAN20195278} Lemma 6. The main idea is to approximate $f$ using $f*\psi_\delta$, where $\psi_\delta$ denotes the density of the Gaussian distribution $N(0,\delta^2)$. And we omit the detail here.

By plugging \eqref{firsttermDeltam+13} and \eqref{secondtermDeltam+1Um} into \eqref{eqn:pickr2} and \eqref{Deltam+12}, we obtain
\begin{align}
\nonumber
& \EE|\Delta^{m+1}_i|^2\\
\nonumber
& \leq\left(1+a\phi_i\right)\EE|\Delta^m_i|^2-2(1+a)h\EE\left[\Delta^m_i\left(\partial_i f(\wx^m)-\partial_i f(x^m)\right)\right]\\
\label{eqn:resultoflem2derivation1}
&+\frac{(1+a)h^2}{\phi_i}\EE\left|\partial_i f(\wx^m)-\partial_i f(x^m)\right|^2+\frac{(1+a)h^3L^2_i}{\phi^2_i}+\left(1+\frac{1}{a}\right)\frac{h^4\left(L^3_i+H^2_i\right)}{\phi^3_i}\,,
\end{align}
where we use $h_i\phi_i=h$. Comparing with~\eqref{eqn:resultoflem2}, we need to set
\[
a=h_i\mu=\frac{h\mu}{\phi_i}<1\,,
\]
where we use $h<\frac{\mu\min\left\{\phi_i\right\}}{8L^2}$. This leads to $1+\frac{1}{a}\leq \frac{2\phi_i}{h\mu}$.
By substituting into \eqref{eqn:Deltam+1new1nonuni}, we obtain
\begin{align*}
\EE|\Delta^{m+1}_i|^2\leq &\left(1+h\mu+\frac{h^2\mu^2}{\phi_i}\right)\EE|\Delta^m_i|^2-2h\EE\left[\Delta^m_i\left(\partial_i f(\wx^m)-\partial_i f(x^m)\right)\right]\\
&+\frac{3h^2}{\phi_i}\EE\left|\partial_i f(\wx^m)-\partial_i f(x^m)\right|^2+\frac{2h^3L^2_i}{\phi^2_i}+\frac{2h^3\left(L^3_i+H^2_i\right)}{\phi^2_i\mu}\,.
\end{align*}
Noting $L_i/\mu>1$, we conclude the lemma.
\end{proof}

The proof of Proposition \ref{prop:rcolmc2} is obtained by summing up all components and applying Lemma~\ref{lem:rcolmc2}.

\begin{proof}[Proof of Proposition \ref{prop:rcolmc2}]
Noting that 
\[
\EE|\Delta^{m+1}|^2=\sum^d_{i=1}\EE|\Delta^{m+1}_i|^2\,,
\]
we substitute using  \eqref{eqn:resultoflem2} to obtain
\begin{align}
\nonumber
\EE|\Delta^{m+1}|^2\leq &\left(1+h\mu+\frac{h^2\mu^2}{\min\{\phi_i\}}\right)\EE|\Delta^m|^2
-2h\EE\left\langle\Delta^m,\nabla f(\wx^m)-\nabla f(x^m)\right\rangle\\
\label{Prop:Deltam+1nonuniform2}
&+\frac{3h^2}{\min\{\phi_i\}}\EE\left|\nabla f(\wx^m)-\nabla f(x^m)\right|^2
+\frac{4h^3}{\mu}\sum^d_{i=1}\frac{\left(L^3_i+H^2_i\right)}{\phi^2_i}\,.
\end{align}

The second and third terms in the right-hand side of this bound  can be controlled by $\EE|\Delta^m|^2$, as follows.
By convexity, we have
\begin{equation}\label{Prop:Deltam+1second2}
\EE\left\langle \Delta^m,\nabla f(\wx^m)-\nabla f(x^m)\right\rangle\geq \mu\EE|\Delta^m|^2\,.
\end{equation}
By the $L$-Lipschitz property, we have
\begin{equation}\label{Prop:Deltam+1third2}
\EE\left|\nabla f(\wx^m)-\nabla f(x^m)\right|^2\leq L^2\EE|\Delta^m|^2\,.
\end{equation}
By substituting \eqref{Prop:Deltam+1second2} and \eqref{Prop:Deltam+1third2} into \eqref{Prop:Deltam+1}, and using $\mu<L$, we have
\begin{equation}\label{eqn:itrcolmcnonuni2}
\EE|\Delta^{m+1}|^2\leq \left(1-h\mu+\frac{4h^2L^2}{\min\{\phi_i\}}\right)\EE|\Delta^m|^2+\frac{4h^3}{\mu}\sum^d_{i=1}\frac{\left(L^3_i+H^2_i\right)}{\phi^2_i}\,.
\end{equation}
Since $h<\frac{\mu\min\left\{\phi_i\right\}}{8L^2}$, we obtain \eqref{eqn:itrcolmc22}.
\end{proof}

\section{Proof of Proposition \ref{prop:badexampleW22}}\label{sec:proofofthmbadexample}
\begin{proof}[Proof of Proposition \ref{prop:badexampleW22}]
For this special target distribution $p$, the objective function is $f(x)=\sum^d_{i=1}\frac{|x_i|^2}{2}$. With $\alpha = 0$ and $\phi_i = 1/d$, we have: $x^{m+1}_i = x^m_i$ for all $i\neq r^m$ and
\[
x^{m+1}_{r^m}=(1-dh)x^m_{r^m}+\sqrt{2dh}\xi^{m}\,.
\]
Therefore for all $i=1,2,\dotsc,d$, we have
\begin{align}
\nonumber
\EE|x^{m+1}_i|^2&=\frac{1}{d}\EE\left(|x^{m+1}_i|^2\middle| r^m=i\right)+\left(1-\frac{1}{d}\right)\EE\left(|x^{m+1}_i|^2\;\middle|\; r^m\neq i\right)\\
\nonumber
&=\frac{1}{d}\EE\left(|(1-dh)x^m_i+\sqrt{2dh}\xi^{m}|^2\middle| r^m=i\right)+\left(1-\frac{1}{d}\right)\EE\left(|x^{m}_i|^2\right)\\
\label{variancecalculate1}
&=\left(1-2h+dh^2\right)\EE|x^{m}_i|^2+2h
\end{align}
where we use $\EE_{\xi}\left|x^m_i-dhx^m_i+\sqrt{2dh}\xi^m\right|^2=(1-dh)^2|x^m_i|^2+2dh$ in the last equation. 
By summing \eqref{variancecalculate1} over $i$, we obtain
\[
\EE|x^{m+1}|^2=\left(1-2h+dh^2\right)\EE|x^{m}|^2+2dh\,.
\]
Using it iteratively, and considering $\EE|x^0|^2=3d$, we have:
\[
\begin{aligned}
\EE|x^m|^2&\geq 3d\left(1-2h+dh^2\right)^m+\left(1-\left(1-2h+dh^2\right)^m\right)\frac{2dh}{2h-dh^2}\\
&= d\left(1-2h+dh^2\right)^m+\frac{2d}{2-dh}+2d\left(1-\frac{1}{2-dh}\right)\left(1-2h+dh^2\right)^m\\
&\geq d\left(1-2h\right)^m+\frac{2d}{2-dh}\,,
\end{aligned}
\]
where we use $dh\leq 1$ in the last inequality.

Since
\begin{equation*}
W(q_m,p)\geq \Bigl(\int |x|^2q_m(x)\rd x\Bigr)^{1/2} - \Bigl( \int|x|^2p(x)\rd x\Bigr)^{1/2}=\Bigl(\int |x|^2q_m(x)\rd x\Bigr)^{1/2}-\sqrt{d}\,,
\end{equation*}
we have
\[
\begin{aligned}
W(q_m,p)\geq \Bigl( \int |x|^2q_m(x)\rd x\Bigr)^{1/2}-\sqrt{d}&\geq \frac{d\left(1-2h\right)^m+\frac{2d}{2-dh}-d}{\sqrt{d\left(1-2h\right)^m+\frac{2d}{2-dh}}+\sqrt{d}}\\
&\geq \frac{\sqrt{d}}{3}\left(1-2h\right)^m+\frac{d^{3/2}h}{6}\\
&\geq \exp\left(-2mh\right)\frac{\sqrt{d}}{3}+\frac{d^{3/2}h}{6}\,,
\end{aligned}
\]
where in the last inequality we use
\[
\sqrt{d\left(1-2h\right)^m+\frac{2d}{2-dh}}+\sqrt{d}\leq 3\sqrt{d}.
\]
Therefore, we finally prove~\eqref{eqn:badexampleW2bound2}.
\end{proof}
\end{appendix}
\end{document}